\documentclass[twoside]{article}
\usepackage{booktabs}
\usepackage{multirow}
\usepackage{amsfonts}
\usepackage{bbm}
\usepackage{algorithm}
\usepackage{algorithmic}
\usepackage{amssymb}
\usepackage{hyperref}
\usepackage[table]{xcolor}
\usepackage{graphicx}
\usepackage{soul}
\usepackage{placeins}
\usepackage{verbatim}
\usepackage{lipsum}
\usepackage[round]{natbib}
\usepackage{amsthm}
\usepackage{xurl}
\usepackage{thmtools,thm-restate}
\usepackage[accepted]{aistats2025}
\begin{document}
\runningauthor{Albert Tseng, Tao Yu, Youngsuk Park}



\twocolumn[

\aistatstitle{Training LLMs with MXFP4}

\aistatsauthor{ Albert Tseng\textsuperscript\textdagger \And Tao Yu \And  Youngsuk Park }

\aistatsaddress{ Cornell University \\ \href{mailto:albert@cs.cornell.edu}{\texttt{albert@cs.cornell.edu}} \And  AWS AI \\ \href{mailto:taou@amazon.com}{\texttt{taou@amazon.com}} \And AWS AI \\ \href{mailto:pyoungsu@amazon.com}{\texttt{pyoungsu@amazon.com}} } ]

\begin{abstract}
Low precision (LP) datatypes such as MXFP4 can accelerate matrix multiplications (GEMMs) and reduce training costs. 
However, directly using MXFP4 instead of BF16 during training significantly degrades model quality. 
In this work, we present the first near-lossless training recipe that uses MXFP4 GEMMs, which are $2\times$ faster than FP8 on supported hardware.
Our key insight is to compute unbiased gradient estimates with stochastic rounding (SR), resulting in more accurate model updates.
However, directly applying SR to MXFP4 can result in high variance from block-level outliers, harming convergence.
To overcome this, we use the random Hadamard tranform to theoretically bound the variance of SR.
We train GPT models up to 6.7B parameters and find that our method induces minimal degradation over mixed-precision BF16 training.
Our recipe computes $>1/2$ the training FLOPs in MXFP4, enabling an estimated speedup of $>1.3\times$ over FP8 and $>1.7\times$ over BF16 during backpropagation.
\end{abstract}

\section{Introduction}

\begin{figure}[t]
\centering
\includegraphics[width=\linewidth]{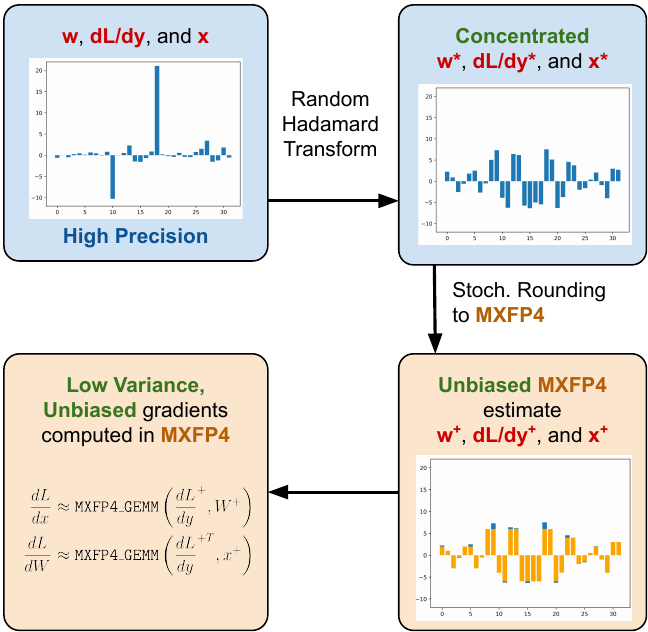}
\caption{Our method uses stochastic rounding (SR) to compute unbiased gradients and the random Hadamard transform to bound the variance of SR. This enables us to perform more accurate model updates with MXFP4 in the backward pass, enabling a speedup of $>1.3\times$ over FP8 and $>1.7\times$ over BF16.}
\label{fig:overview}
\end{figure}

The latest large language models (LLMs) have billions of parameters that are trained on trillions of tokens, making them incredibly expensive to train. 
For example, training Llama 3.1 405B required $3\times 10^{24}$ floating point operations (FLOPs), or over 10000 GPUs for multiple months \citep{llama3}. 
Recent hardware accelerators have started supporting low precision floating point ($\le$ 16 bit) matrix mulitiplications (GEMMs).
Compared to 32-bit GEMMs, hardware-accelerated low precision (LP) GEMMs run at significantly higher throughputs.
For example, FP8 GEMMs can be $4\times$ faster than FP32 GEMMs and also more energy efficient \citep{blackwell}.

Since LLM training is compute bound in matrix multiplications, LP GEMMs can accelerate training.
Almost all modern LLMs are trained with 16 bit GEMMs \citep{llama2, llama3}, and some even use FP8 GEMMs \citep{msamp}.
Using 16 bit GEMMs halves the cost of matrix multiplications and improves end-to-end throughput by almost $2\times$ \citep{mpiclr}.
%Even further, FP8 is emerging as a viable training datatype that can achieve another $30-50\%$ end-to-end improvement over 16-bit training \citep{msamp}.
%\tao{not 50\% for end to end, see MS-AMP for example, 40\% might be more appropriate} \youngsuk{1) CITE, 2) from where? e.g., reduced communication} improvement in end-to-end throughput over 16-bit training.
However, there is no free lunch with low precision training.
Reducing the GEMM precision increases quantization distortion and can cause numerical instability.
%As such, most FP8 training recipes require fine-grained scaling \textit{within} a tensor to minimize degradation over 16 bit training.
%FP8 recipes also use different FP8 constructions during the forward and backward passes to minimize quantization distortion. \youngsuk{cite/mention papers like https://arxiv.org/html/2405.18710v1, pointing (naive) FP8 solution is unstable.}

%\let\thefootnote\relax\footnotetext{Correspondence to \href{mailto:albert@cs.cornell.edu}{albert@cs.cornell.edu} and \href{mailto:taou@amazon.com}{taou@amazon.com}}.

To counteract these issues, the recently introduced Microscaling (MX) family of datatypes uses a shared blockwise scale across multiple floating point numbers \citep{ocpmx}.
For example, MXFP4 uses an INT8 scale $s$ for every contiguous block $v$ of 32 FP4 numbers to represent $2^{s-1}v$, where 1 is the exponent bias for FP4.
This scale enables a significantly wider range of representable numbers at the cost of an extra $8/32 = 0.25$ bits per entry.
%In hardware, such as on NVIDIA Blackwell GPUs, MXFP4 GEMMS can be computed at almost twice the throughput as FP8 GEMMs \tao{references}, meaning that this built-in scale is essentially free in compute bound settings.
However, MX alone is not enough to enable lossless low precision training with FP4.
As we show in Section \ref{sec:experiments}, directly using MXFP4 in even only the backward pass of decoder linear layers significantly degrades model quality.

In this work, we introduce two techniques that enable near-lossless distributed training with MXFP4.
Our method hinges on computing low-variance, unbiased gradient estimates that enable more accurate model updates.
First, we use stochastic rounding to compute unbiased GEMMs.
Then, we use a memory-bound construction of random Hadamard transform to reduce the effect of outliers and theoretically bound the variance of SR, aiding convergence.
We apply our method to decoder linear layers and show that it incurs minimal degradation over BF16 mixed precision training when pretraining GPT models up to 6.7B parameters.
Our recipe computes over half the training FLOPs in MXFP4 and can significantly accelerate pretraining.
In summary, we:

%This enables us to perform accurate model updates in expectation, 
% Our method hinges on the observation that minibatched training methods compute unbiased gradient estimates, so further computing the gradient estimate with LP GEMMs should be fine as long as the final estimate is still unbiased with low variance. \tao{here better mention lightly the introduced techniques don't cause additional overheads, which are important to give some impression}
%Indeed, our method lets us compute the gradients of decoder linear layers during training with minimal final model degradation.
%Since decoder linear layer backpropagation consists of over half the training FLOPs, our MXFP4 recipe can significantly accelerate LLM training \citep{flopcount, megatronpaper}.

\begin{itemize}
\item Introduce a MXFP4 training recipe that uses stochastic rounding and the random Hadamard transform compute unbiased, low-variance gradient estimates during backpropagation.
\item Pretrain GPT models up to 6.7B and show that our training recipe closes the MXFP4-BF16 gap to $<0.1$ validation perplexity. %\tao{can we use a percentage gap?}\albert{afaik perplexity isn't measured in terms of perc gap since its on a log scale}
\item Show that our RHT and SR constructions add minimal overhead to MXFP4 GEMMs, giving a theoretical speedup of $>1.3\times$ and $>1.7\times$ over a FP8 and BF16 backward pass, respectively.% \tao{maybe combine the 1st and 3rd bullet as they are directly related}
\end{itemize}

\section{Background and Related Works}
%\youngsuk{Background is too long, but we combined with related work here? then better cite thoroughly}
\subsection{Low Precision Datatypes and IEEE 754 Floating Point Numbers}

Traditionally, low precision (LP) datatypes refer to datatypes that use significantly fewer than 32 bits to represent a single number. 
For example, FP16 uses 16 bits to represent floating point numbers.
While there are many different low precision datatypes, including stateful ones \citep{qtip}, a certain subset has been standardized under the IEEE754 floating point (FP) \citep{IEEE754}. These datatypes often come with hardware acceleration for compute-bound workloads.

IEEE floats (Table \ref{tab:fpdtype}) are defined with 1 sign bit, $e$ exponent bits, and $m$ mantissa bits.
In shorthand, a $1+m+e$ bit datatype is written as E$e$M$m$.
%For example, E5M2 refers to the FP8 construction that uses 5 mantissa bits and 2 exponent bits.
The actual ``normal'' value represented by an IEEE float with sign bit $S$, mantissa $M$, and exponent $E$ is %\tao{this is for normal numbers only, subnormal numbers you don't have the 1, right? yes}
$$(-1)^S (1+M) 2^{E - \texttt{bias}},$$ 
where \texttt{bias} is a datatype-dependent integer exponent bias offset specified by \cite{IEEE754}. 
% \youngsuk{what is the logic of deciding bias? give reference?}
% \albert{I'm not sure the logic, but its from the IEEE spec}\tao{the bias is set so that E-bias can cover positive and negative exponents evenly, it's usually to be $\lfloor (max_E-1)/2 \rfloor$ like E4M3, $max_E=15$, and bias=7; and FP4, $max_E=3$, and bias=1}
% \youngsuk{right. thanks! let's shortly mention that.}
% \albert{but this is irrelevant to our paper? it doesn't matter how bias is decided nor is it something we care about}
This exponent-mantissa construction means FP datatypes are scale-invariant with respect to quantization signal-to-noise ratio (SNR) bar over/underflow \citep{graphcore}.
%While floats have wide hardware support, they are not very efficient at using information.
%\tao{I don't understand this claim, FP32 is good.}.
%\albert{Efficient means how low of a distortion you can get for a certain bitrate. A 4 bit trellis quantizer will get much lower distortion than FP4 on almost any input source. FP32 achieves low distortion on an absolute scale, but you can get the same distortion with much fewer bits.}
%For example, FP4 represents $-0$ and $+0$ separately; this wastes $4 - \log_2(15) = 0.093$ bits. 
%\youngsuk{good information, but necessary to mention this in context of what? to justify MX?}
%\albert{^ just to point out that FP numbers are suboptimal}
%\tao{this seems to justify FP4 is bad, then why still use FP4? I don't see a reason to use FP4 as a bad example here. we use fp4 because it is hardware supported and fast. The "bad" part is to show how we can still make training work with a bad dtype.}
%\tao{is this example illustrating the higher-than-desirable distortion?}
% No it just shows that we waste 1 entry out of 16, which naturally increases distortion since there are fewer codepoints.

\begin{table}[t]
\centering
\caption{Common HW supported FP datatypes.}
\label{tab:fpdtype}
\begin{tabular}{@{}ccccc@{}}
\toprule
\multirow{2}{*}{Name} & \multicolumn{4}{c}{Bits}    \\ \cmidrule(lr){2-5}
                      & Total & Sign & Exponent & Mantissa                                                                                    \\ \midrule
FP64                  & 64    & 1    & 11   & 52                                                               \\
FP32                  & 32    & 1    & 8    & 23                                                                 \\
FP16                  & 16    & 1    & 5    & 10                                                                          \\
BF16                  & 16    & 1    & 8    & 7                                                                      \\
FP8 \tiny{E4M3}       & 8     & 1    & 4    & 3                                                                        \\
FP8 \tiny{E5M2}       & 8     & 1    & 5    & 2                                                                       \\
FP4                   & 4     & 1    & 2    & 1                                                               \\ \bottomrule
\end{tabular}
\end{table}

\subsection{LLM Training}

The most common way to train a LLM involves computing a loss function, computing the gradient of the loss function with respect to the model parameters, and then updating the parameters with gradient information.
For example, when pretraining a decoder-only LLM, one might use an autoregressive cross-entropy-based loss and the AdamW optimizer \citep{llama2, llama3}.
While the exact training setup may differ, the core bottlenecks of training are the compute-bound forward and backward passes that calculate the loss and gradients, respectively.
Within these two components, the majority of the FLOPs are in the linear layers -- at 30B parameters, over 90\% of the FLOPs are in the linear layers \citep{flopcount}.

The forward pass for a linear layer with input dimension $n$ and output dimension $m$ computes $y = xW^T + b$, where $W \in \mathbb{R}^{m \times n}$ is a parameter matrix and $b \in \mathbb{R}^{m}$ is an optional bias term.
To backpropagate through a linear layer, we need to calculate the gradient of $y$ with respect to $x, W,$ and $b$.
These are given by $\frac{dL}{dx} = \frac{dL}{dy}W$, $\frac{dL}{dW} = \frac{dL}{dY}^Tx$, and $\frac{dL}{db} = \mathbbm{1}\frac{dL}{dy}$, where $\mathbbm{1}$ is the all-ones vector and $\frac{dL}{dy}$ is the backprop output from the previous (going backwards) operation in the chain rule \citep{backprop}.
Each linear layer requires 3 computationally intensive matrix multiplications ($xW^T, \frac{dL}{dx}$, and $\frac{dL}{dW}$), 2 of which are in the backward pass.
%\youngsuk{any good citation on backward operation? even under nonlinear, layernorm, distributed setting? (I have my own derivation though). We can defer detailed derivation to them.}

\subsection{Mixed Precision Training}

One way to accelerate training is with ``mixed precision'' (MP) training.
In MP, parameters are kept in high precision and GEMM operands are converted to a LP datatype for a LP GEMM.
%The quantization step can also be implemented as a ``copy buffer'' (such as in Megatron-LM \citep{megatronpaper}) with separate FP32 and LP copies of the parameters. 
%The outputs of the matmuls are kept in high precision for accurate parameter updates.
MP is a simple way to achieve the throughput benefits of LP datatypes since quantization usually has minimal overhead.
%This is slightly numerically different but does not affect convergence in practice.
End to end, BF16 MP is often $>70\%$ faster than FP32 training \citep{megatronpaper}.
However, quantization introduces distortion in the GEMM operands and thus outputs.
Since the forward and backward passes all happen in low precision, both the loss and the model updates can deviate from their ``true'' values.
At low bitrates $\ll 16$, distortion can degrade model quality and even cause divergence, necessitating advanced training recipes.
For example, FP8 MP recipes typically use E4M3 (more precision) in the forward pass and E5M2 (more range) in the backward pass due to the different properties of gradients, weights, and activations \citep{msamp, te}.
%\youngsuk{Q: was it also applied to our case?}
% no because there is only one fp4 formulation and I screwed up the fp8 experiments by not modifying TE correctly.

At 4 bits, quantization distortion becomes even more difficult to manage.
\cite{int4training} train smaller non-GPT transformers with INT4 GEMMs by using the non-randomized Hadamard transform in the forward pass and leverage score sampling (LSS) in the backward pass.
Since LSS introduces additional overhead, they were only able to achieve an end-to-end speedup of 30\% over FP16, which is on par with FP8 mixed precision training \citep{msamp}. %\tao{are they also lossless? they don't test large transformers so idk, but iirc no for smaller transformers (albert)}
We are also aware of a concurrent work by \cite{msfp4} that trains LLMs with FP4.
There, the authors train billion parameter GPT models with FP4 in both the forward and backward pass by using a differentiable gradient estimator and keeping outliers in high precision, resulting in a perplexity gap of $>0.5$.
Since their work was released after our paper went through the review process, we reserve a full comparison for future work.

% Their method results in a significantly larger perplexity gap ($>0.5$) than ours at the cost of. 
% It is also not clear how much overhead high precision outliers add in hardware, since even one outlier in a hardware GEMM tile requires ``activating'' the entire tile.
% For a sufficiently bad outlier distribution, an additional sparse GEMM could cost more than the FP4 GEMM itself.
% In contrast, our method has fixed overhead that is significantly lower than the cost of a 4 bit GEMM.

\subsection{Stochastic Rounding}
%\youngsuk{let's have mathematical description?}
Mixed precision requires quantizing from a higher precision tensor to a LP tensor at every step -- this opens up flexibility in how the actual quantization happens.
The canonical ``nearest rounding'' (NR) method rounds each high precision number to its closest representable value in the LP datatype \citep{IEEE754}.
However, NR is not \textit{unbiased}, which we later show to be detrimental to low precision training.
One way to achieve unbiased rounding is with ``stochastic rounding'' (SR), which randomly rounds a number to a representable value in the LP datatype so that, in expectation, the rounded number equals the original number \citep{sr}.
%Since this introduces additional stochasticity, most SR implementations minimize variance by only rounding between the two closest representable LP values surrounding the number \citep{sr}.
% \youngsuk{can reduce last couple of sentences via citation and math below.}

SR can be implemented efficiently through \textit{dithering}, which adds random uniform noise to the input number and then performs NR \citep{sr}.
For example, Amazon's Trainium line of chips can perform SR with dithering while adding less than 2\% overhead to a BF16 GEMM. % \citep{trainium, muhamed2023training}
%\tao{this number still seems large, it should be less? in the end of sec4.2, it's 2\% overhead?}\albert{this is SR, not RHT} .
Equation \ref{eqn:unifdither} describes SR with dithering for a uniform integer quantizer; the non-uniform case requires modifying the noise scale but is otherwise essentially the same.
\begin{align}
\label{eqn:unifdither}
\delta &\sim \mathcal{U}(-0.5, 0.5) \\
\text{SR}_{\text{dither}}(x) &= 
\begin{cases} 
      \lfloor x \rfloor & x+\delta < \lfloor x \rfloor + \frac{1}{2} \\
      \lceil x \rceil & x+\delta \ge \lfloor x \rfloor + \frac{1}{2} 
\end{cases}
\end{align}
% \youngsuk{Optional: want to use $\mathrm{SR}(x;D)$ with datatype $D$ and rounding $\lfloor x \rfloor_D$ w.r.t. D rather than $SR_{dither}(x)$? We can use $\mathrm{SR}(x;FP4) <- sotchastic round to fp4$ in Algo 2 }
% \youngsuk{here $x$ is an integer?}
% \youngsuk{Optional, can we define NE similarly with $n=0$ and use it for Algo 1 description $P_i = quantize\_to\_element\_format(V_i / X)$?}
% \albert{SR\_dither refers to stochastic rounding implemented with dithering. x is a real number. This formula describes SR with dithering to the integers. This formula does not work for nonuniform (eg floating point) datatypes. You have to do some bitshifting there that is handled in hardware}

SR can also be used anywhere where numbers are quantized.
%For example, when adding two numbers of the same precision but with vastly different magnitudes, information from the smaller number is ``lost.''
%This is problematic when applying model updates from the optimizer.
For example, near the end of training, the model update norm is much smaller than the parameter norm and information in low precision updates can be``lost'' \citep{collage}.
Here, stochastic rounding can be used to preserve the update \textit{in expectation}, which uses less memory than keeping a high precision copy of the parameters. %\tao{we should advance the mention Trn1 and its efficient support of SR earlier here (now in sec4 last paragraph).}

\subsection{Microscaling (MX) FP Formats}

The recently introduced microscaling floating point family of datatypes builds upon IEEE floats by adding a groupwise scale to a base IEEE float \citep{ocpmx}.
% While MX could be used with a base integer datatype, this work focuses on MX floats since hardware MX implementations usually only support floats.
This scale allows a MXFP tensor to take on a wider range of values without significantly increasing the total bitrate, with the caveat that entries in a group should be roughly the same magnitude for the scale to be useful.
In practice, MX scaling is more important as the base datatype bitrate decreases.
Whereas FP8 E4M3 has a dynamic range of $\frac{448}{2^{-9}} = 2.3\times10^6$, FP4 has a dynamic range of $\frac{6}{0.5} = 12$. 
% \youngsuk{what is demonomator here? related to bias?}
% \albert{The denominator is the smallest normal representable values} \tao{think we refer the `dynamic range` as the max/min values that can be represented with the format, what's the definition of dynamic range here? max value of two FP8 mul?}
% \youngsuk{yes, formally defining dynmaic range would be helpful for newbies.}
% \albert{dynamic range is a standard metric and term used to refer to min/max. https://en.wikipedia.org/wiki/Dynamic\_range. People can google it if they want to.}
MX scaling enables MXFP4 to represent a much wider range of values \textit{across blocks}.
%\tao{without scaling factors, FP8 E4M3 max value is 448, right? how does $2^{-9}$ chip in here? also for the later FP4 case 2^-9 is the smallest fp8 normal and 0.5 is the smallest fp4 normal} 

%\youngsuk{mention MX is in more need for lower precision like fp4, than fp8 or higher, to compensate the usage of smaller number of bits to represent numbers.}
The core hardware-supported MXFP formats generally follow similar patterns.
Scales are shared across contiguous entries in memory (usually 32), and quantizing a scalar tensor to a MX tensor depends on the largest element in each group \citep{ocpmx,nvptx}.
Algorithm \ref{alg:float2mx} describes the ``reference'' algorithm for quantizing a scalar tensor to MX, which can be implemented efficiently on modern AI accelerators \citep{nvcutlass}.
Algorithm \ref{alg:float2mx} scales each group based on its maximum magnitude element and then performs nearest rounding to obtain a MX tensor. 
% This means that quantizing a tensor stored in row-major format (groups span rows) may result in a \textit{different} MX tensor than quantizing the same tensor stored in column-major format (groups span columns). \albert{32 contiguous, row/column depends on r/c major}\tao{can you point me a reference for this?}\albert{algorithm 1. almost all the comments are answered by algorithm 1, which is directly from the OCP paper}
% %\tao{nice catch, which one do we use? and any ablation?}
%\albert{we use row major since that's what torch does. I didn't run an ablation, if we have time we can try in the appendix}

\begin{algorithm}[t]
\caption{Convert vector of scalar floats $V \in \texttt{HP\_DTYPE}^k$ to an MX block
$\{X, P \in \texttt{LP\_DTYPE}^k\}$ (from \citep{ocpmx})}
\label{alg:float2mx}
\begin{algorithmic}[1]
\REQUIRE $\texttt{emax}_\texttt{elem}$ = exponent of largest normal in \texttt{LP\_DTYPE}, $k=32$ for hardware support.
%\youngsuk{input better be dtype (e.g., e2m1-FP4)? which will be used to decide emax and quantization function}
%\albert{input is any real number, emax is a property of fp4}
\STATE \texttt{shared\_exp} $\gets \lfloor \log_2(\max_{i}(|V_i|)) \rfloor - \texttt{emax}_\texttt{elem}$
\STATE $X \gets 2^{\texttt{shared\_exp}}$
\FOR{$i=1$ to $k$}
  \STATE $P_i$ = \texttt{quantize\_to\_LP}$(V_i / X)$
\ENDFOR
\RETURN $X,\>\{P_i\}_{i=1}^k$
\end{algorithmic}
\end{algorithm}

\section{Training with MXFP4}

%MXFP4 is appealing for a variety of reasons. 
%Although hardware FP4 support has only recently begun, high volume chips such as NVIDIA's Blackwell GPUs implement FP4 with microscaling \citep{blackwell}.
%On actual hardware, initial benchmarks have shown that MXFP4 GEMMs can offer close to $2\times$ the real-world throughput of FP8 GEMMs.
The rest of this paper describes our approach that enables near-lossless \textit{training} with MXFP4-accelerated GEMMs.
Although our paper focuses on MXFP4, our analysis also applies to other low precision datatypes such as MXINT4.
We chose MXFP4 due to its relevance and hardware support on the latest accelerators.
To the best of our knowledge, MXFP4 has only been successfully used for near-lossless inference \citep{mxtrain,bwmlperf}.
Although certain works have achieved near-lossless training with MXFP4 weights, these require the activations and gradients to kept in higher precision.
These recipes run at the throughput of the higher precision operand, making them slower than pure-FP4 recipes.

%\albert{make the reasons appear earlier in the paragraph}
Our method hinges on obtaining unbiased, low-variance gradient estimates with pure-MXFP4 GEMMs in the backward pass, enabling more accurate model updates. 
Since the backward pass consists of $>1/2$ training FLOPs, our recipe can significantly accelerate training without reducing the representational power of the model from LP forward passes \citep{kumar2025scaling}.
To do this, we first modify the OCP MX quantization algorithm to perform unbiased quantization with scaling and stochastic rounding. 
Then, we show that by first transforming the GEMM operands with a memory-bound construction of the random Hadamard transform (RHT) \textit{before quantization}, we can bound the variance of the GEMM output.
Our method adds minimal overhead while significantly improving the quality of trained models, making MXFP4 practical for training.

%Instead, we show that with the random Hadamard transform (RHT) and stochastic rounding (SR), we can use MXFP4 in the backward pass and BF16 in the forward pass without noticeable degradation over full BF16 training.
%\youngsuk{mention (exact) datatype used for forward}
%The RHT concentrates the operands before quantization to reduce distortion and underflow.
%SR enables unbiased MXFP4 GEMMs and thus gradient estimates.

% \begin{table}[h]
% \centering
% \renewcommand{\tabcolsep}{2pt}
% \caption{Mean squared error (MSE) and maximum absolute error (MAE) of an MXFP4 matmul where both operands have outliers. ``+RHT'' denotes applying the RHT to both operands before multiplication, and ``+SR'' denotes using Algorithm \ref{alg:sr2mx} instead of Algorithm \ref{alg:float2mx} when quantizing to MXFP4. The RHT reliably improves both metrics, whereas SR improves both metrics in expectation.}
% \begin{tabular}{@{}ccccccc@{}}
% \toprule
%     & MXFP4   & +RHT    & +SR (1 iter) & +RHT+SR (1 iter) & +SR (5 iter) & +RHT+SR (5 iter) \\ \midrule
% MSE & 1.22E-3 & 7.45E-4 & 2.58E-3      & 1.37E-3          & 5.05E-4      & 2.76E-4          \\
% MAE & 7.43    & 4.71    & 8.56         & 2.29             & 4.37         & 0.92             \\ \bottomrule
% \end{tabular}
% \end{table}

\subsection{Unbiased Quantization to MXFP4}

Algorithm \ref{alg:float2mx} describes the ``reference'' MX quantization algorithm to convert a scalar matrix to an MX matrix.
Algorithm \ref{alg:float2mx} finds, for each group of 32 entries, value with the largest magnitude $m=\max_i (|V_i|)$.
% \youngsuk{$m= \log \max V_i$}
% \albert{no, m is just the largest value. floor(log2(m)) is the exponent associated with m}
% \youngsuk{ok!}
Then, it calculates a shared exponent as a function of $m$ and $\texttt{emax}_{\texttt{elem}}$, the largest exponent of a normal number in the base data format.
% \youngsuk{is it the same as e in EeMm? the description could be confusing, what is normal number? may be good to have the example below}
% \albert{no, Ee is the number of exponent bits, its emax is ($2^e$ - 1) - bias}\tao{normal numbers (values abs$>$1) and subnormal numbers (values abs$<$1) are standard terms in FP; here we should make it clear this $\texttt{emax}_{\texttt{elem}}$ is a fixed value of the underlying FP, but not computed from elements in the block}
% \youngsuk{right. why not just define emax as ($2^e$ - 1) - bias, then? is there any case where there is no normal number? in that case what emax value we set?}
% \albert{It may happen depending on how IEEE adds new datatypes. OCP uses emax\_elem so I'm following them.}
For example,  $\texttt{emax}_{\texttt{elem}}= 2$ for FP4 since its maximum normal value is $6 = 2^2*1.5$.

Finally, group elements are normalized by the shared exponent and rounded to the base datatype.
% \youngsuk{general description of MX format can go to section 2.5?}

For MXFP4, line 1 of Algorithm \ref{alg:float2mx} returns $\texttt{shared\_exp} \gets \lfloor \log_2(m)\rfloor - 2$.
Observe that after dividing the entire group by $2^{\texttt{shared\_exp}}$, $m$ becomes 
\begin{equation}
m \gets \frac{m}{2^{\texttt{shared\_exp}}} < \frac{m}{2^{\log_2(m) - 3}} = 8
\end{equation}
%\youngsuk{want to indicate $\overset{\text{fp}4}{\leftarrow }$? or $\text{FP4}\left(\frac{m}{2^{\lfloor \log_2(m)\rfloor - 2}} \right)$ to indicate FP4 transformation of scaled value? or simply use $p_{\max}$ on the LHS?}
Since the maximum representable normal value in FP4 is 6, values scaled to between 6 and 8 will get clipped, making Algorithm \ref{alg:float2mx} inherently biased.
% \youngsuk{isn't $\lfloor \log_2(m)\rfloor - 2=0$ for max input $m=6$, so $\frac{6}{2^{\lfloor \log_2(m)\rfloor - 2}}= 6$ , what am I missing? I may have some misunderstanding on emax. Could you give an example the scale value is actually larger than $6$?}
% \albert{consider 7. floor(log2(7)) = $2 -> 2^0 = 1 -> 7$ gets scaled to 7 $>$ 6}\tao{Note $m$ here is the high precision value before quantization, not after quantization (FP4), so it can be any value larger than 6}
% \youngsuk{right $m$ can come from higher precision. Then, @albert I believe
% \begin{equation}
% m \gets \frac{m}{2^{\texttt{shared\_exp}}} < \frac{m}{2^{\log_2(m) - 3}} = 8
% \end{equation}
% should be correct.
% At the same time, m is lower bounded by $4$. 
% }
% \youngsuk{@albert, let's clarify the following in algo 1 or in the text or both. 1) $V_i$ is any number, 2) emax and quantize-to-element-format is determined by output dtype like fp4. even simply emax$= 2^e-1-bias$ and that long function into dtype($V_i/X)$ where dtype=FP4. techinially speaking that data conversion function should be also part of requirement?}
Although the proportion clipped depends on the input matrix, we can empirically check that for a wide distribution of matrices, roughly 3\% of the entries will get clipped. 

We can make Algorithm $\ref{alg:float2mx}$ unbiased with two simple modifications, both of which can be efficiently implemented in hardware.
First, we scale $V_i/X$ by $3/4$ to prevent clipping.
Then, we use stochastic rounding to quantize $Q'$ to FP4, which gives an unbiased estimate of $Q'$.
Algorithm \ref{alg:sr2mx} summarizes these modifications.
The resulting MX matrix is an unbiased estimate of 3/4 the original matrix.
%Note that without scaling, SR would not be able to give an unbiased estimator for elements $>6$ since there are no FP4 values $>6$.
Since SR is implemented with uniform independent dithering in hardware, the resulting GEMM output is an unbiased estimator of $(3/4)^2=9/16$ of the correct output.
To get an unbiased output, we can simply scale the high precision accumulator output by 16/9.

\begin{restatable}{lemma}{srlemma}
\label{lem:srlemma}
Assume stochastic rounding is implemented with dithering with independent noise. Then, Algorithm \ref{alg:sr2mx} produces a MXFP4 matrix that is an unbiased estimate of $3/4$ its input. Furthermore, Algorithm \ref{alg:rhtbw} with Algorithm \ref{alg:sr2mx} as a subroutine produces an unbiased estimate of $\frac{dL}{dx}$ and $\frac{dL}{dW}$.
\end{restatable}

%These changes allow us to compute unbiased estimators for the weight and activation gradients with MXFP4.
%This is important since the gradient-based optimization algorithms that LLMs are trained with suffer adversely from biased gradient estimates.
%While our experiments show that having a biased gradient estimate is not catastrophic, there is still an empirical benefit to having an unbiased gradient estimate.

%\youngsuk{@albert, can we highlight this finding in the intro/experiments more? It is theoretically more sound, appealing to AISTATS folks and could help other applications, although SR does not contribute too much under RHT in the experiments. }

\begin{algorithm}[t]
\caption{Unbiased quantization of $V \in \texttt{HP\_DTYPE}^k$ to an MXFP4 block
$\{X, P \in \texttt{LP\_DTYPE}^k\}$}
\label{alg:sr2mx}
\begin{algorithmic}[1]
\REQUIRE $\texttt{emax}_\texttt{elem}$ = exponent of the largest normal number in \texttt{LP\_DTYPE}
\STATE \texttt{shared\_exp} $\gets \lfloor \log_2(\max_{i}(|V_i|)) \rfloor - \texttt{emax}_\texttt{elem}$
\STATE $X \gets 2^{\texttt{shared\_exp}}$
\FOR{$i=1$ to $k$}
\STATE $V_i \gets \frac{3}{4} V_i$
\STATE $P_i = \texttt{stochastic\_round\_to\_FP4}(V_i / X)$
\ENDFOR
\RETURN $X,\>\{P_i\}_{i=1}^k$
\end{algorithmic}
\end{algorithm}

\subsection{Bounding the Variance of SR with the Random Hadamard Transform}
\label{sec:smallrht}
The backward pass for a linear layer ($y = xW^T$) requires computing $\frac{dL}{dx} = \frac{dL}{dy} W$ and $\frac{dL}{dW} = \frac{dL}{dy}^T x$.
LLMs have been known to have activation ($x$) and weight ($W$) ``outliers'' as well as sparse gradients ($\frac{dL}{dy}$) \citep{int4training, qs}.
Recall that MXFP4 quantization relies on groupwise statistics such as the largest magnitude element, so blocks with outliers will suffer from high quantization distortion and stochastic rounding variance.

Although Lemma \ref{lem:srlemma} tells us that Algorithm \ref{alg:sr2mx} produces an unbiased estimate of the true GEMM, high variance estimates can still degrade model quality by effectively adding noise to the gradient estimate.
To remedy this, we use the randomized Hadamard transform to concentrate gradients, activations, and weights before quantization, which asymptotically reduces the variance of the GEMM output.

The random Hadamard transform performs $x \gets HSx$, where $x \in \mathbb{R}^{j\times k}, S \in \{\pm 1\}^k$ (a random sign vector), and $H$ is the $k$-dimensional Hadamard matrix \citep{rht}. 
Hadamard matrices are recursively defined orthogonal matrices that satisfy the following:
\begin{equation}
\label{eqn:hadrec}
H_n = \frac{1}{2^{n/2}}\begin{bmatrix}
H_{n-1} & H_{n-1} \\
H_{n-1} & -H_{n-1}
\end{bmatrix},
\end{equation}
where $H_1 = \begin{bmatrix} 1 \end{bmatrix}$. 
Since both $H$ and $diag(S)$ are orthogonal, the RHT is fully invertible.
This means that we can apply the RHT to GEMM operands without inverting the RHT -- that is, $(HSA)^T(HSB) = A^TB$.

\begin{restatable}{theorem}{rhtvar}
\label{thm:rhtvar}
Let $A$ and $B$ be two size-$b$ vectors $\in \mathbb{R}^b$, and let $\mathcal{Q}$ perform Algorithm \ref{alg:sr2mx}. Then, the variance of $\mathcal{Q}(A)^T\mathcal{Q}(B)$ is $\mathcal{O}(b\Delta^4\|A\|_\infty\|B\|_\infty)$ and the variance of $\mathcal{Q}(HSA)^T\mathcal{Q}(HSB)$ is, with probability $\ge (1-\epsilon)^2$, $\mathcal{O}(\Delta^4\|A\|\|B\|\log(2b/\epsilon))$, where the largest gap between two consecutive representable points in $\mathcal{Q}$'s quantizer is $\Delta$. 
%\tao{and what is $\Delta$ here, do we define in the paper anywhere?}
% \tao{what is $\sigma$}
% \youngsuk{I read the proof. yeah, the most non-solid part is about the assumption $\mathcal{O}(\sigma)$ put uniformly over $b$ elements. To me, 
% 1) should be okay to hide $\sigma$ part as it is a common factor over two variance, only mentioning in the proof or 2) defining it more being aligned with $N(0, \sigma^2)$ used in figure 2, where $\sigma$ was used diffirently from one in the proof. But, given the timeline we have, just go for it, if we really want, can modify after arxiv upload.}
% \albert{this is standard practice that we used in quip. You need to mention the quantizer distortion otherwise using a better quantizer would have no effect, which doesn't make sense.}
\end{restatable}

Theorem \ref{thm:rhtvar} tells us that the variance of a MX matrix multiplication with respect to stochastic rounding is linear in the product of the largest magnitude elements in the operands. 
Applying the RHT to a vector effectively concentrates it to have a sub-Gaussian tail distribution. 
From \citet{qs}, we know that
\begin{equation}
\label{eqn:hadtail}
\mathbb{P}\left( |e_iHSx| \ge a\right) \le 2\exp\left(\frac{-a^2k}{2\|x\|^2}\right),
\end{equation}

letting us bound the the variance of the SR GEMM in Theorem \ref{thm:rhtvar}. 
Specifically, applying the RHT reduces the variance from a linear dependence on blocksize to a log-dependence on blocksize, albeit with the $L_2$ norm of the input instead of the $L_\infty$ norm.

We can verify this empirically by measuring the variance of a SR GEMM with and without the RHT. 
Figure \ref{fig:rhtvar} shows the mean variance of $\mathcal{Q}(A)^T\mathcal{Q}(B)$ vs. $\mathcal{Q}(HSA)^T\mathcal{Q}(HSB)$ over 4K samples of $A, B \in \mathbb{R}^{b} \sim \mathcal{N}(0, I)$ with proportion $p$ outliers from $\mathcal{N}(0, 5I)$, where $\mathcal{Q}$ performs Algorithm \ref{alg:sr2mx}.
That is, $A, B \sim \mathcal{N}(0, I) + \mbox{Bernoulli}(p)*\mathcal{N}(0, 5I)$.
As expected from Theorem \ref{thm:rhtvar}, the variance grows much slower as a function of $b$ with the RHT vs. without.
% \youngsuk{where is $p$ part demonstrated?}\albert{x axis}

\begin{figure}[t]
\includegraphics[width=\linewidth]{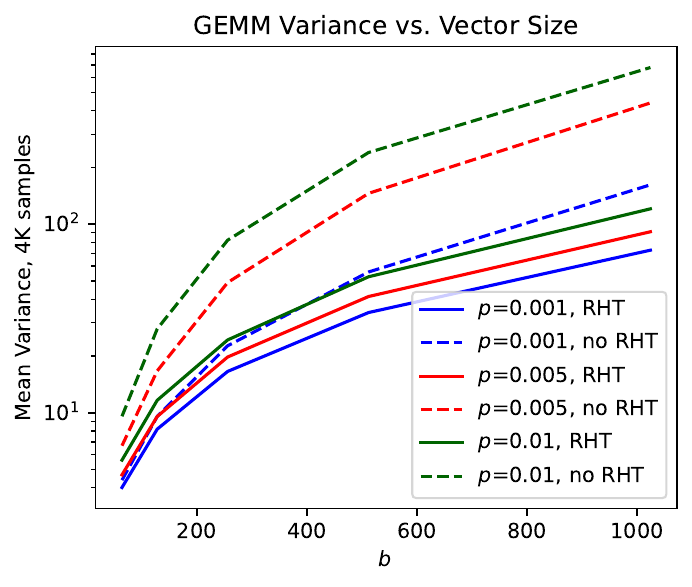}
\caption{Mean variance of $\mathcal{Q}(A)^T\mathcal{Q}(B)$ vs. $\mathcal{Q}(HSA)^T\mathcal{Q}(HSB)$ over 4K samples of $A, B \in \mathbb{R}^{b} \sim \mathcal{N}(0, I)$ with proportion $p$ outliers from $\mathcal{N}(0, 5I)$. $\mathcal{Q}$ performs Algorithm \ref{alg:sr2mx}. Variance with the RHT grows much slower than without.
%\youngsuk{@alber, examining dependence on $b$ could be more important than $\sigma$?}
%\albert{I will plot that too}
}
\label{fig:rhtvar}
\end{figure}

\begin{algorithm}[t]
\caption{MXFP4 linear layer (no bias) backward pass with the random Hadamard transform.}
\label{alg:rhtbw}
\begin{algorithmic}[1]
\REQUIRE Gradient of output $\frac{dL}{dy} \in \mathbb{R}^{b\times m}$, activations $x \in \mathbb{R}^{b \times n}$, weights $W \in \mathbb{R}^{m \times n}$, block size $g \le 256, 32|g, g|m, g|n$.
\STATE $H \gets \text{Hadamard matrix } H_b \in \mathbb{R}^{m \times m}$.
\STATE Sample random sign vector $S \in \{\pm1\}^b$.
\STATE $G' \gets \left(\left(\frac{dL}{dy}\right).\texttt{view}\left(\frac{bm}{g}, g\right)\right)\texttt{diag}(S)H$
\STATE $W' \gets H^T \texttt{diag}(S) \left(W.\texttt{view}\left(g, \frac{nm}{g}\right)\right)$
\STATE $GT' \gets \left(\left(\frac{dL}{dy}^T\right).\texttt{view}\left(\frac{bm}{g}, g\right)\right)\texttt{diag}(S)H$
\STATE $X' \gets H^T \texttt{diag}(S)\left(x.\texttt{view}\left(\frac{bn}{g}, g\right)\right)$
\STATE $\frac{dL}{dx} \gets \texttt{MXFP4\_GEMM}(G', W')$
\STATE $\frac{dL}{dW} \gets \texttt{MXFP4\_GEMM}(GT', X')$
\\\COMMENT{Where \texttt{MXFP4\_GEMM} forms MX groups along the reduction dimension and uses either Algorithm \ref{alg:float2mx} or \ref{alg:sr2mx} to quantize to MXFP4.}
\IF{Using Algorithm \ref{alg:sr2mx}}
\STATE $\frac{dL}{dx} \gets \frac{16}{9} \frac{dL}{dx}$
\STATE $\frac{dL}{dW} \gets \frac{16}{9} \frac{dL}{dW}$
\ENDIF
\RETURN $\frac{dL}{dx}, \frac{dL}{dW}$
\end{algorithmic}
\end{algorithm}

However, the RHT is not free.
First, observe that when computing $\frac{dL}{dW} \approx \mathcal{Q}(HS\frac{dL}{dy})^T\mathcal{Q}(HSx)$, the RHT ``mixes'' along the batch dimension.
In data-parallel settings (e.g. FSDP \citep{fsdp} or ZeRO-3 \citep{zero}) where activations are sharded across GPUs, the full RHT would require expensive cross-GPU communication.
Even with fast interconnects, this would immediately bottleneck gradient computation.
Second, although Equation \ref{eqn:hadrec} admits an $O(n \log n)$ time matrix-vector product algorithm, the RHT step occurs in high precision.
Reducing this overhead is critical -- if the RHT is slower than a FP4 matmul, one should just use FP8 instead.

To solve these problems, we apply the RHT as a dense matrix multiplication over a small number of MX blocks, which makes it \textit{memory bound} in the GEMM operands (see Table \ref{tab:e2e}).
Specifically, let the RHT block size be $g, 32 | g$. %\youngsuk{what is 32 $|$ g?} \albert{32 divides g. this is standard notation} \youngsuk{thanks!}
Applying this block-wise RHT as a dense matmul gives a runtime of $O((b+m)ng)$ and IO cost of $O(bn+nm+bm)$.
Since modern AI accelerators have high compute to memory ratios, this ``blockwise'' RHT is memory bound when $g \lessapprox 256$. 
Algorithm \ref{alg:rhtbw} summarizes how we use the RHT in the backward pass of a linear layer. 
Since $g$ is smaller than the sequence length of any reasonably large model, Algorithm \ref{alg:rhtbw} works as a drop-in replacement for a linear layer even in data-parallel settings.
Furthermore, although lines 3-6 are written out for clarity, an efficient implementation could fuse them into lines 7 and 8, reducing costly memory accesses. %\tao{do we have an ablation table of RHT block size - running time or throughput (in TP=1 and 8)? that would explain better our claim; looks like table 5? can we link to it}\albert{I can link the table. TP8 is too annoying to set up}

The tradeoff to doing this blockwise RHT is that equation $\ref{eqn:hadtail}$ depends on $g$ ($k$ in the equation) -- the higher $g$ is, the tighter the concentration will be. 
However, in practice, we observe $g = 64$ is sufficient to get a tight distribution and MX can handle scale differences across blocks.
Finally, note that this construction also lets us use \textit{any} random orthogonal transformation. %\tao{isn't every orthogonal matrix can be ran with a dense matmul?}
We chose the RHT since it is fast to randomize (by sampling a single $g$-dim sign vector) and has good concentration, but other matrices could work as well.

\section{Experiments}
\label{sec:experiments}

Our main experiments focus on pretraining GPT 345M, 1.3B, and 6.7B \citep{gpt}.
We follow prior low precision training works and train for at least 20 billion tokens, which is sufficient to determine overall training performance on a longer full-scale run \citep{msamp}.
We use the Megatron-LM codebase to train our models \citep{megatronpaper}, the publicly available GPT2 Wikipedia dataset \citep{gpt2wiki}, and the bit-accurate Microsoft \texttt{microxcaling} library for MX emulation \citep{microxcaling}.
Since pretraining is expensive, we stopped certain experiments short when it was clear they did not match BF16.
Our analysis below uses validation perplexity from a holdout set, but we observe the same behavior with training perplexity. 
Training perplexity plots and additional experiments can be found in the Appendix.
%Finally, although the results here show MXFP4 only in the backward pass, we also tried using MXFP4 in the forward pass as well.
%We were unable to match BF16 with MXFP4 in the forward pass (see Appendix) and leave that for future work.

\subsection{GPT Pretraining Results}

\begin{table}[t]
\caption{Final losses for GPT models trained on the GPT2 Wikipedia corpus. All models were trained with BF16 mixed precision for the forward pass.}
\label{tab:bf16fw}
\centering
\renewcommand{\tabcolsep}{5pt}
\begin{tabular}{@{}ccccc@{}}
\toprule
Params. & Toks. & Bwd. Prec.   & \begin{tabular}[c]{@{}c@{}}Train.\\ Loss\end{tabular} & \begin{tabular}[c]{@{}c@{}}Val.\\ Loss\end{tabular} \\ \midrule
\rowcolor[HTML]{EFEFEF} 
345M    & 33B   & BF16         & 2.58                                                  & 2.49                                                \\
\rowcolor[HTML]{EFEFEF} 
345M    & 33B   & MXFP4        & 2.73                                                  & 2.60                                                \\
\rowcolor[HTML]{EFEFEF} 
345M    & 33B   & MXFP4+RHT    & 2.60                                                  & 2.51                                                \\
\rowcolor[HTML]{EFEFEF} 
345M    & 33B   & MXFP4+RHT+SR & 2.60                                                  & 2.51                                                \\
1.3B    & 42B   & BF16         & 2.28                                                  & 2.32                                                \\
1.3B    & 42B   & MXFP4        & 2.44                                                  & 2.40                                                \\
1.3B    & 42B   & MXFP4+RHT    & 2.30                                                  & 2.33                                                \\
1.3B    & 42B   & MXFP4+RHT+SR & 2.29                                                  & 2.32                                                \\
1.3B    & 42B   & MXFP4+SR     & 2.29                                                  & 2.32                                                \\
\rowcolor[HTML]{EFEFEF} 
1.3B    & 210B  & BF16         & 2.06                                                  & 2.29                                                \\
\rowcolor[HTML]{EFEFEF} 
1.3B    & 210B  & MXFP4+RHT    & 2.09                                                  & 2.31                                                \\
\rowcolor[HTML]{EFEFEF} 
1.3B    & 210B  & MXFP4+RHT+SR & 2.07                                                  & 2.29                                                \\
\rowcolor[HTML]{EFEFEF} 
1.3B    & 210B  & MXFP4+SR     & 2.08                                                  & 2.29                                                \\
6.7B    & 21B   & BF16         & 2.04                                                  & 2.27                                                \\
6.7B    & 21B   & MXFP4+RHT    & 2.05                                                  & 2.28                                                \\
6.7B    & 21B   & MXFP4+RHT+SR & 2.08                                                  & 2.27                                                \\ \bottomrule
\end{tabular}
\end{table}

% \begin{figure*}[t]
% \centering
% \includegraphics[width=0.32\linewidth]{figs/345m.pdf}
% \includegraphics[width=0.32\linewidth]{figs/1b3.pdf}
% \includegraphics[width=0.32\linewidth]{figs/6b7.pdf}
% \caption{GPT validation perplexity curves (with BF16 forward pass) for (L) 345M, (C) 1.3B and (R) 6.7B. With the RHT and SR, our MXFP4 approach can match the performance of BF16 in the backward pass.}
% \label{fig:valppl}
% \end{figure*}

\begin{figure}[t]
\includegraphics[width=\linewidth]{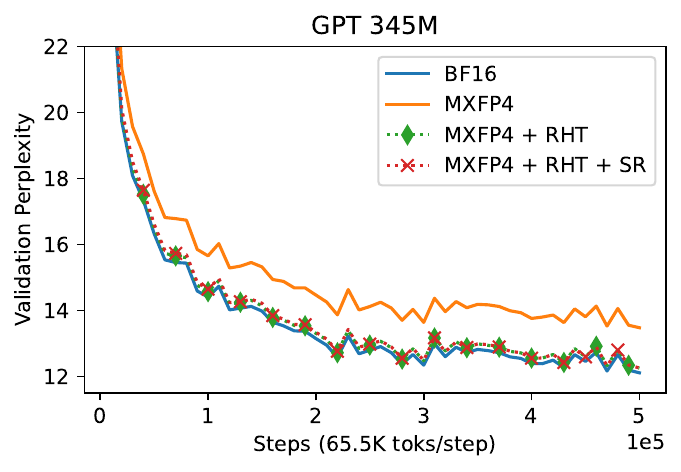}
\caption{GPT 345M validation perplexity curves with BF16 forward pass. With RHT and SR, MXFP4 can match the performance of BF16 in the backward pass.}
\label{fig:345mbf16}
\end{figure}

\begin{figure}[t]
\includegraphics[width=\linewidth]{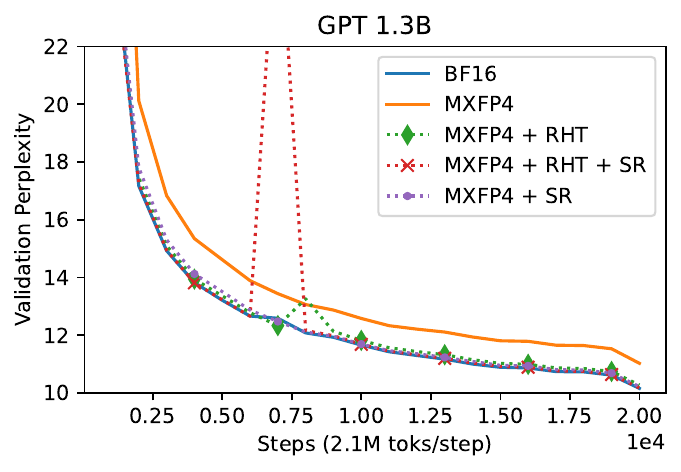}
\caption{GPT 1.3B validation perplexity curves with BF16 forward pass. With RHT and SR, MXFP4 can match the performance of BF16 in the backward pass.}
\label{fig:1b3bf16}
\end{figure}

% \begin{figure}[ht]
% \includegraphics[width=\linewidth]{figs/1b3_init.pdf}
% \caption{Initial convergence for GPT 1.3B. Stochastic rounding's variance makes it converge slower than BF16 and the RHT variants.}
% \label{fig:1b3init}
% \end{figure}

\begin{figure}[t]
\includegraphics[width=\linewidth]{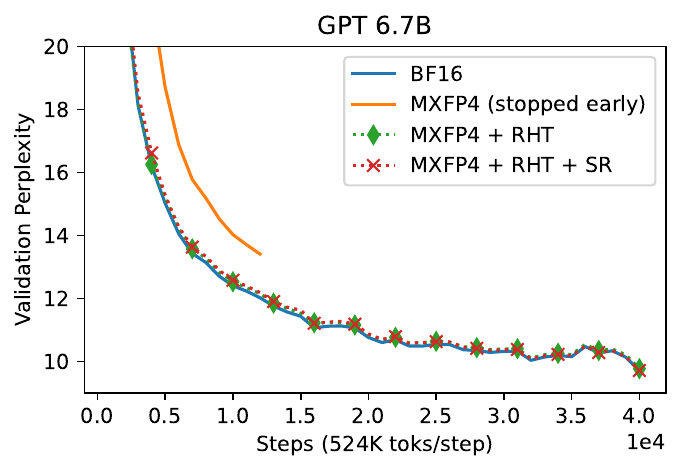}
\caption{GPT 6.7B validation perplexity curves with BF16 forward pass. With RHT and SR, MXFP4 can match the performance of BF16 in the backward pass. The MXFP4-only run was stopped early to save resources.}
\label{fig:6b7bf16}
\end{figure}

\begin{figure*}[t]
\includegraphics[width=\linewidth]{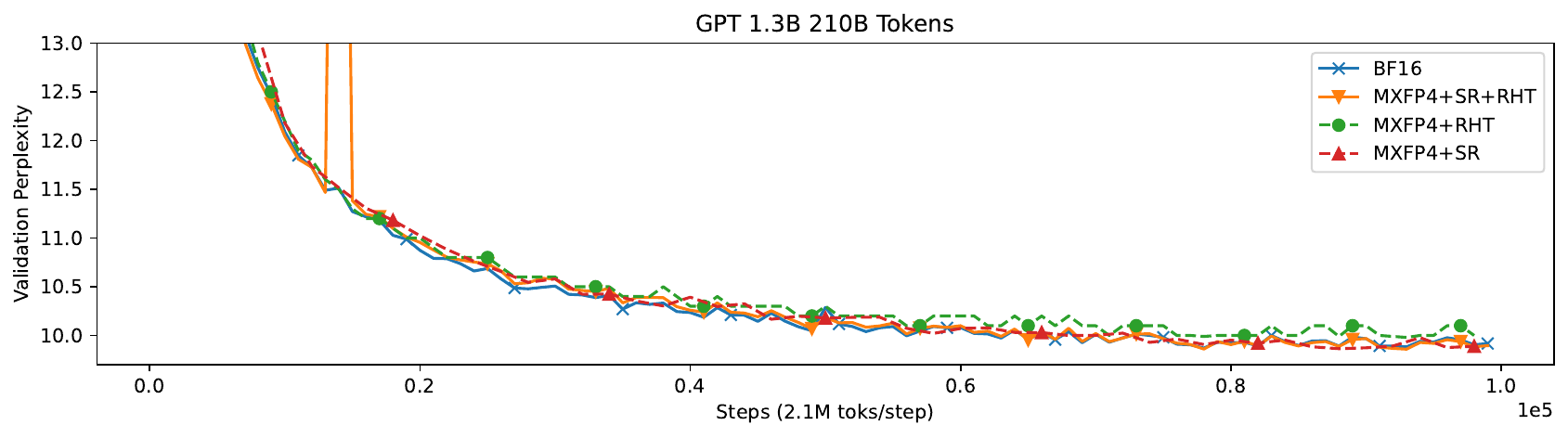}
\caption{Validation perplexity for training GPT 1.3B for 210 billion tokens, or $5\times$ longer than in Table \ref{tab:bf16fw}. All experiments used BF16 in the forward pass and the specified backward precision.}
\label{fig:200b}
\end{figure*}

\begin{table}[t]
\centering
\renewcommand{\tabcolsep}{3pt}
\begin{tabular}{@{}cccccc@{}}
\toprule
Model          & ArcC & ArcE & PiQA & BoolQ & Wino \\ \midrule
BF16  & 23.1 & 49.2 & 60.5 & 53.3  & 52.0       \\
MXFP4$\star$ & 22.2 & 47.8 & 61.3 & 59.6  & 49.6       \\ \midrule
BF16 \small{\textsc{Tulu V2}} & 25.6 & 50.6 & 62.7 & 59.6 & 51.6 \\
MXFP4$\star$ \small{\textsc{Tulu V2}} & 25.9 & 49.9 & 62.9 & 60.5 & 51.8 \\ \bottomrule
\end{tabular}
\caption{GPT 6.7B model trained on 20B tokens before and after Tulu V2 fine-tuning. Both BF16 and our MXFP4+RHT+SR (MXFP4$\star$) model exhibit similar performance before and after fine-tuning. 
%\tao{don't we have the downstream evals for 200B tokens 1.3B model? this MXFP4* feels WORSE compared to bf16 before finetuning}
%\albert{no, because one of the 200b checkpoints got deleted for some reason. mxfp4 is better in 2 evals and slightly worse on the other three. these are the same numbers from the rebuttal.}
}
\label{tab:finetune}
\end{table}

\begin{table}[t]
\centering
\begin{tabular}{@{}cccccc@{}}
\toprule
BW Pass  & BF16  & g=32  & g=64  & g=128 & g=256 \\ \midrule
Val. PPL & 11.89 & 12.02 & 12.01 & 11.98 & 11.98 \\ \bottomrule
\end{tabular}
\caption{Validation perplexity for training GPT 345M on 33B tokens with various RHT blocksizes. Increasing the RHT block size improves performance by reducing the variance of stochastic rounding.}
\label{tab:rhtbs}
\end{table}

\begin{table*}[t]
\centering
\begin{tabular}{@{}ccccccccc@{}}
\toprule
BW Pass   & FP16  & \begin{tabular}[c]{@{}c@{}}INT8 \\ \small{\textsc{no RHT}}\end{tabular} & \begin{tabular}[c]{@{}c@{}}INT4 \\ \small{\textsc{no RHT}}\end{tabular} & \begin{tabular}[c]{@{}c@{}}+ RHT\\ \small{\textsc{g=64}}\end{tabular} & \begin{tabular}[c]{@{}c@{}}+ RHT\\ \small{\textsc{g=128}}\end{tabular} & \begin{tabular}[c]{@{}c@{}}+ RHT\\ \small{\textsc{g=256}}\end{tabular} & \begin{tabular}[c]{@{}c@{}}+ RHT\\ \small{\textsc{g=1024 dense}}\end{tabular} & \begin{tabular}[c]{@{}c@{}}+ RHT\\ \small{\textsc{g=1024} $\mathcal{O}(n \log n)$}\end{tabular} \\ \midrule
E2E tok/s & 46983 & 55469                                                                   & 67306                                                                   & 64335                                                                 & 64171                                                                  & 63979                                                                  & 61186                                                                         & 62640                                                                                           \\
BW tok/s  & 72563 & 94688                                                                   & 133952                                                                  & 123056                                                                & 122734                                                                 & 121823                                                                 & 112299                                                                        & 120495                                                                                          \\ \bottomrule
\end{tabular}
\caption{Throughput for a FP16 forward pass and specified backward pass of a Llama 2 70B decoder layer. Measured on a NVIDIA A100; see Section \ref{sec:overhead} for more details. Since the A100 can perform INT4 GEMMs $4\times$ faster than FP16 GEMMs, 
%\youngsuk{FP16 or BF16? FP16, its llama 2b}
these numbers represent the expected speedup of MXFP4 on supported hardware.}% \tao{this is all measured with TP=1, right?}\albert{yes, 1 gpu}}
\label{tab:e2e}
\end{table*}

Table \ref{tab:bf16fw} and Figure \ref{fig:200b} show our main results with using BF16 in the forward pass and various MXFP4 constructions in the backward pass.
We ablate on using the RHT only, which produces a biased but reduced-distortion GEMM, and RHT and SR, which produces an unbiased, lower varianace GEMM.
For GPT 1.3B, we also measure the performance of MXFP4+SR only, which gives an unbiased but higher variance GEMM.
All experiments use Megatron-LM's mixed precision implementation with separate FP32 master weights and BF16 parameter copies.
For the backward pass for decoder linear layers, the BF16 and gradients are quantized to MXFP4.
Experiments with the RHT use $g=64$, which mixes across 2 MX blocks. 

Table \ref{tab:bf16fw} shows that for shorter runs (20-40 billion tokens), using either the RHT or SR with MXFP4 is sufficient to achieve near-lossless training all tested model sizes.
However, Figure \ref{fig:200b} shows that for longer runs (210 billion tokens), having an unbiased gradient estimator is necessary to maintain performance.
Whereas using the RHT only results in an $\approx 0.1$ perplexity gap, using stochastic rounding (with or without the RHT) results in \textit{no validation perplexity gap}.

Figures \ref{fig:345mbf16}, \ref{fig:1b3bf16} and \ref{fig:6b7bf16} show the validation perplexity curves for the experiments in Table \ref{tab:bf16fw}.
At all scales, the MXFP4+RHT+SR curve closely tracks BF16. 
In contrast, although the final performance of MXFP4+SR matches MXFP4+SR+RHT, MXFP4+SR exhibits slower initial convergence than BF16 and MXFP4+RHT+SR. %\tao{this is indeed to observe from the figure?}\albert{yes, the purple line is initially worse}.
We suspect that this is due to loss of gradient information without the RHT. 
Although using only SR will give an unbiased gradient estimator, small values will still get stochastically flushed to 0 (Equation \ref{eqn:unifdither}), resulting in loss of gradient information.
In contrast, the RHT transforms the gradient to a different space.
This reduces variance and also significantly reduces the probability that a single gradient entry in the original space will be set to 0.
To verify this, Table \ref{tab:rhtbs} shows an ablation on the RHT block size -- increasing the block size improves quality.
%\albert{add black flushing experiment}

These figures also include curves for using pure MXFP4 (no RHT and no SR) MP in the backward pass.
Using only MXFP4 (the orange curve) results in significant degradation and a large perplexity gap at all sizes.
Even further, if we consider that FP8 is near-lossless \citep{msamp} vs. BF16 and is only an estimated 30-40\% slower end-to-end than pure MXFP4, then pure MXFP4 isn't even ``worth it.''
For a fixed amount of wall clock time, simply training with FP8 for fewer steps would give a better model than using pure MXFP4.
In contrast, our techniques close the gap to BF16 and FP8, making MXFP4 practical for training.
Our techniques are also compatible with FP8 forward passes (Figure \ref{fig:fp8fw}, more details in Appendix), further pushing the speed-quality tradeoff curve.

To further evaluate our MXFP4 models, we ran zeroshot evaluation for downstream tasks on our 20B token GPT 6.7B models. 
Both the BF16 and MXFP4+RHT+SR models perform around the same.
To test how well these models can be fine-tuned, we fine-tuned them using the publicly-available Tulu V2 dataset (657M tokens) and codebase \citep{tulu2}. 
We used the hyperparameters in the Tulu V2 codebase and trained for 5 epochs with BF16/FP32 mixed precision. 
The BF16 model reached a final training perplexity of 1.96, and the MXFP4+RHT+SR model 1.98.
Like before finetuning, both models achieve similar zeroshot performance, indicating that they are of similar quality.
Table \ref{tab:finetune} summarizes these results.

\begin{figure*}[t]
\centering
\includegraphics[width=0.4\linewidth]{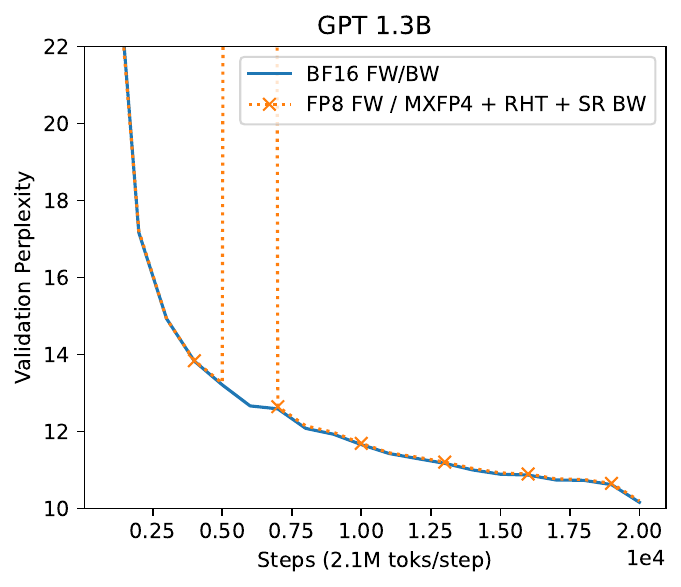}
\hspace{1cm}
\includegraphics[width=0.4\linewidth]{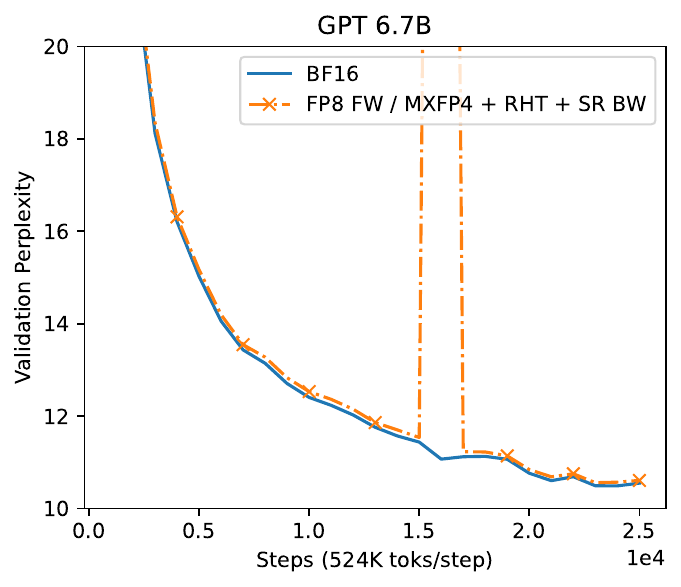}
\caption{GPT 1.3B and 6.7B perplexity curves with a FP8 forward pass, our MXFP4 backward pass, and the same settings as Figures \ref{fig:1b3bf16} and \ref{fig:6b7bf16}. Our method is compatible with FP8 forward passes for additional speedups. See Appendix for details.
% \tao{this figure margin should be adjusted for better layout}\albert{I lost the code to regenerate the plots and don't want to remake the plots}
}
\label{fig:fp8fw}
\end{figure*}

\subsection{Overhead Calculations}
\label{sec:overhead}

%\albert{cite msamp and te and nv benchmarks and make clear this is a gemm cost estimate}
The goal of MXFP4 training is to achieve a wall-clock time speedup over FP8 training. 
Unfortunately, we do not have access to FP4 hardware yet so we cannot measure empirical wall-clock speedups over FP8.
However, we can estimate the overhead of the RHT and stochastic rounding with proxy benchmarks.

Our RHT construction operates on a small ``tile'' in the operand and is memory bound, so we can conceivably fuse it with the MXFP4 GEMM and avoid writing its output to memory.
We can estimate the performance of this setup in two ways.
First, we measured the overhead of RHT-GEMM kernels for \textit{FP8}.
% , which we do have hardware for.
Specifically, we timed $A\texttt{.to(E4M3)}B\texttt{.to(E4M3)}^T$, $A \in \texttt{BF16}^{n \times k}, B \in \texttt{BF16}^{m \times k}$ with and without the RHT along the $k$ dimension.
We generated Triton \citep{triton} kernels with \texttt{torch.compile} \citep{torch}, an RHT size of $g=64$, and benchmarked 7B and 70B-sized matrices: $(m,n,k)=(32768, 8192, 8192)$ and $(16384, 28672, 28672)$.
On a NVIDIA H100 GPU, the RHT adds 9.7\% overhead for the 7B-sized setup and 1.6\% for the 70B-sized setup.
Assuming MXFP4 has twice the throughput of FP8, these numbers would double to $19.4\%$ and $3.2\%$, respectively, which is still faster than a FP8 GEMM.

Second, we measured the overhead of the RHT on the HuggingFace implementation \citep{hf} of single Llama 2 70B decoder layer decoder layer on a NVIDIA A100 GPU.
Specifically, we report the end-to-end tokens per second for computing the forward pass in FP16 and backward pass in either FP16 or INT4, which has the same hardware speedup ($4\times$) on the NVIDIA A100 vs. FP16 as MXFP4 has on modern hardware.
We also include INT8 as a proxy for the expected speedup of a FP8 backward pass.
We use a batch size of 4 sequences with 4K tokens each (16K tokens/batch), Flash Attention 2 \citep{fa2}, \texttt{torch.compile}, and the CUTLASS INT4 and INT8 GEMM kernels. 
We were unable to use CUDA graphs since the HuggingFace implementation is not compatible with CUDA graphs; we expect CUDA graphs to improve speedup ratios by masking kernel launch overhead.

Table \ref{tab:e2e} summarizes these results.
End to end with a FP16 FW pass, an INT4+RHT backward pass is over $40\%$ faster than a FP16 backward pass and over $20\%$ faster than an INT8 backward pass.
If we only consider the backward pass, INT4+RHT is $\approx 70\%$ faster than a FP16 backward pass and $\approx 30\%$ faster than a INT8 backward pass.
The HuggingFace Llama implementation is not known to be fast, so a more efficient implementation would achieve better INT4 and INT8 speedups over FP16.
Table \ref{tab:e2e} also shows that the RHT adds less than $5\%$ E2E overhead and is memory bound in the operands until $g \approx 256$. 
Interestingly, the recently released $\mathcal{O}(n \log n)$ HadaCore kernel \citep{hadacore} recovers most of the dense GEMM penalty at $g=1024$, but is still slower than smaller $g$.
%Finally, compared to an INT8 backward pass, an INT4 + RHT backward pass is still $30\%$ faster. \tao{the last sentence is redundant, isn't it, we just mentioned 30\% early in the paragraph.}

To measure the overhead of stochastic rounding, we used an Amazon Trainium 1 chip (EC2 \textit{Trn1} instance), which is one of the few widely available chips that has dedicated stochastic rounding hardware \citep{trainium}.
Our experiments show that for most matrix sizes, using SR to quantize GEMM operands from FP32 to BF16 adds less than 2\% overhead over the BF16 GEMM itself.
Assuming a 4$\times$ increase in GEMM throughput when going from BF16 to FP4, this would mean SR adds less than 10\% overhead.

\section{Conclusion}

While hardware support for low precision datatypes continues to advance, it is becoming increasingly difficult to train with these datatypes without suffering from significant model degradation.
In this work, we demonstrate the first MXFP4 training recipe that achieves near-lossless model quality vs. FP32/BF16 mixed precision training.
Our method hinges on computing low variance, unbiased gradient estimates for decoder linear layer, which enables us to make more accurate model updates.
%Although MXFP4 shares a scale across blocks of FP4 num#bers, this shared scale is not enough to compensate for FP4's distortion.
To do this, we propose using stochastic rounding (SR) and the random Hadamard transform (RHT). %to calculate the gradients of decoder linear layers in the backward pass.
Stochastic rounding produces unbiased gradient estimates, and the RHT reduces the variance of SR and the chance of losing gradient information from underflow.
Our experiments pretraining GPT models up to 6.7B show that both the RHT and SR are crucial for near-lossless MXFP4 training.
Finally, our benchmarks show that our method can be implemented with minimal overhead, giving an estimated 30\% speedup over FP8 and 70\% speedup over BF16 in the backward pass.

\section*{Acknowledgements}

We thank Chris De Sa for valuable feedback. We also thank Yida Wang and George Karypis for their support within AWS AI Research.
% \tao{@Youngsuk, do we need to ack Yida?}\youngsuk{done}
\vfill
\pagebreak
\pagebreak

\bibliographystyle{plainnat}
\bibliography{sample_paper.bib}

\clearpage

\thispagestyle{empty}
\onecolumn
\aistatstitle{Training LLMs with MXFP4: \\
Supplementary Materials}

\section{Additional Results}

\subsection{FP8 Forward Pass Results}

This section contains experiments using mixed-precision FP8 in the forward pass and MXFP4 in the backward pass.
Prior works have shown that mixed-precision FP8 forward and backward passes can be close to lossless over mixed-precision BF16 training \citep{msamp, nvte}.
To test if MXFP4 backward passes are still practical with FP8 forward passes, we trained GPT 1.3B and 6.7B models with NVIDIA's TransformerEngine (TE) FP8 (E4M3) implementation \citep{nvte} in the forward pass and our MXFP4 formulation in the backward pass. 
We did not test GPT 345M since TE FP8 already had a $>0.1$ validation perplexity gap vs. BF16 in our experiments.
For GPT 6.7B, since we did not have access to FP8-capable hardware with fast interconnects necessary for tensor-parallel training, we emulated FP8 matrix multiplications by dequantizing FP8 GEMM operands into BF16 and performing a BF16 GEMM.
While this is not bit-accurate vs. a FP8 GEMM, the relative error of the output is $\approx 0.3\%$ for random Gaussian inputs.
Furthermore, this is essentially how PyTorch emulates FP8 GEMMs \citep{torch}.
At both model scales, we find that FP8 forward passes and MXFP4 backward passes are sufficient to essentially match BF16 training.

\FloatBarrier
\begin{figure}[h!]
\centering
\includegraphics[width=0.48\linewidth]{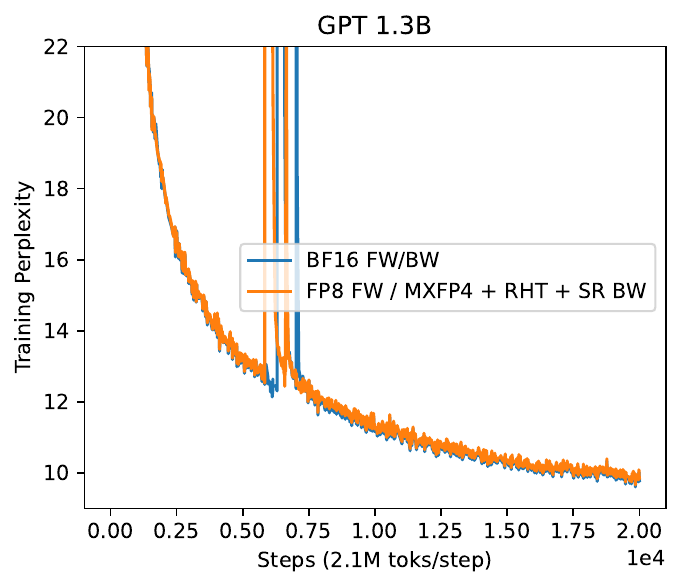}
\caption{Training perplexity curves for GPT 1.3B for 33 billion tokens. Using FP8 in the forward pass and MXFP4 in the backward pass does not result in noticeable degradation.}
\end{figure}

\begin{figure}[h!]
\centering
\includegraphics[width=0.48\linewidth]{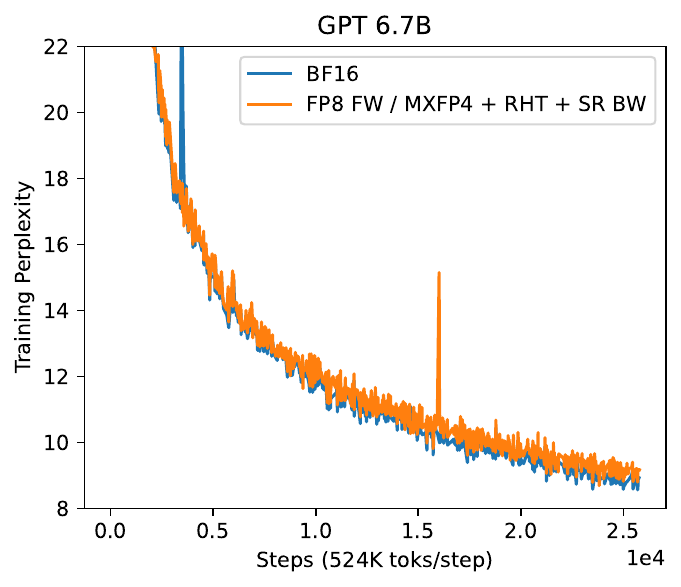}
\caption{Training perplexity curves for GPT 6.7B for the first 13 billion tokens of a 20 billion token run. Due to time constraints and the cost of training a 6.7B parameter model, we were unable to include the full run. Like 1.3B, using FP8 in the forward pass and MXFP4 in the backward pass does not result in noticeable degradation.}
\end{figure}
\FloatBarrier

\subsection{GPT 345M Validation Curves with Stochastic Rounding Only}

This plot is the same as those from the main body except that it includes an experiment with stochastic rounding only (no RHT). Like 1.3B, SR starts off ``worse'' than the RHT variants but is able to match their performance at the end of the training run. 
% \tao{but you didn't put the outlier parts at the beginning of training, actually the purple line is lower than the rest at the beginning from this figure} \albert{That's from the curve smoothing, there was some missing data from TB for the beginning of the plot}

% \FloatBarrier
% \begin{figure}[h!]
% \centering
% \includegraphics[width=\linewidth]{figs/345m_train.pdf}
% \caption{Training perplexity for training GPT 345M for 33 billion tokens. All experiments used BF16 in the forward pass and the specified backward precision. All curves were smoothed with a Savitzky-Golay filter to enhance readability.}
% \end{figure}

\begin{figure}[h!]
\includegraphics[width=\linewidth]{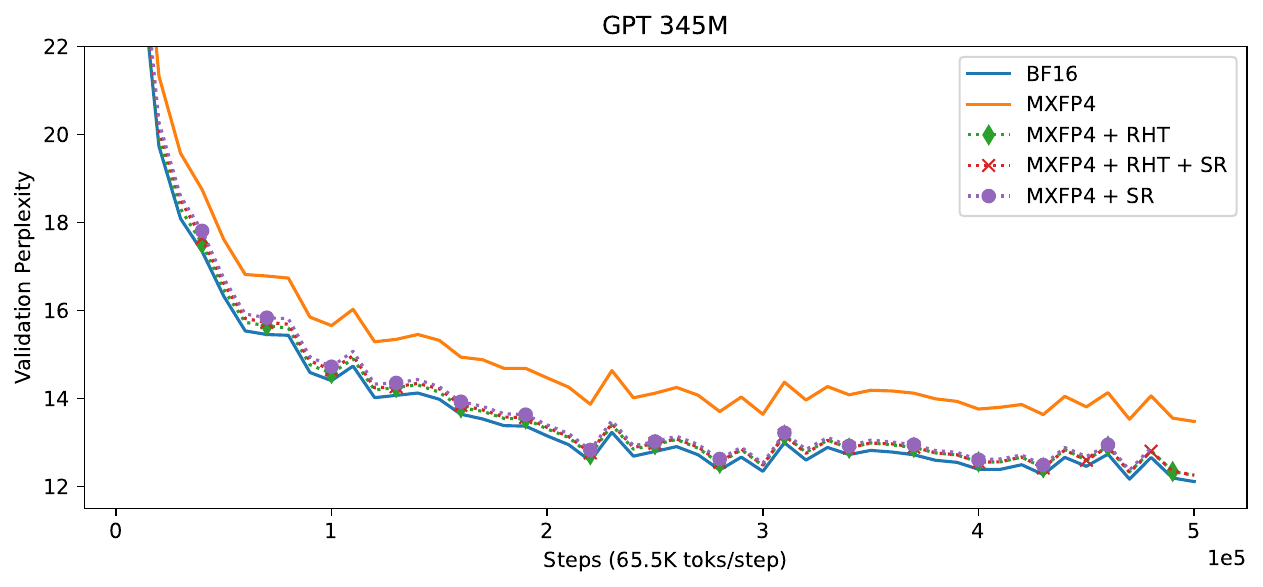}
\caption{Validation perplexity for training GPT 345M for 33 billion tokens. All experiments used BF16 in the forward pass and the specified backward precision. This plot is the same as Figure \ref{fig:345mbf16} in the main body except that it adds an experiment with MXFP4+SR only. }
\end{figure}

\FloatBarrier
\subsection{Training Curve for GPT 1.3B}
\FloatBarrier
\begin{figure}[h!]
\includegraphics[width=\linewidth]{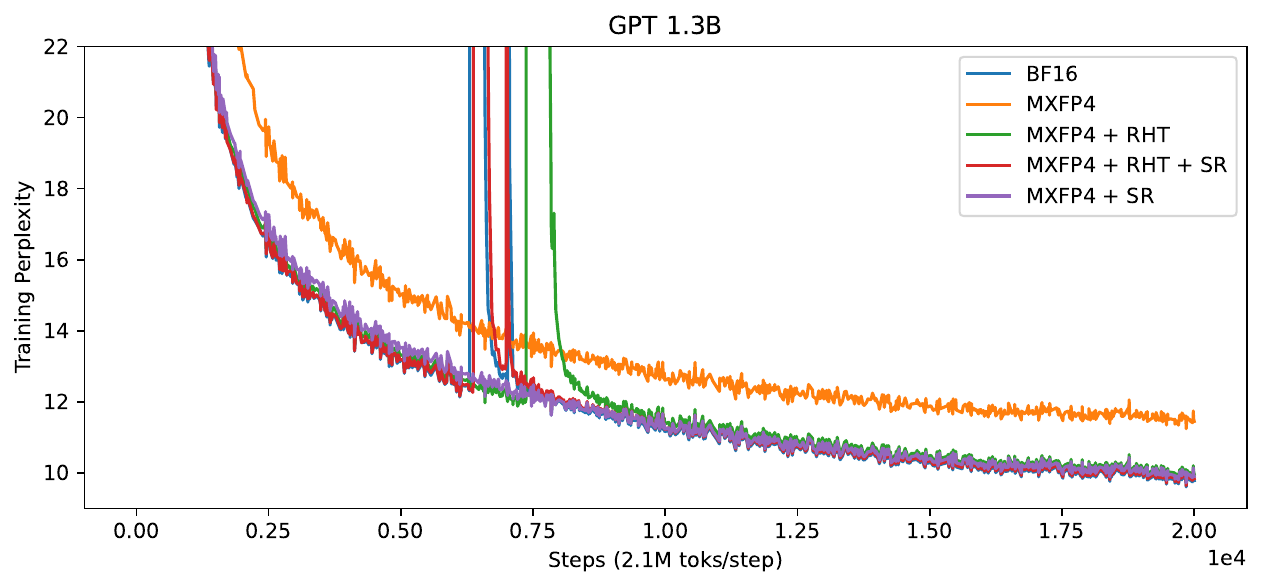}
\caption{Training perplexity for training GPT 1.3B for 40 billion tokens. All curves used BF16 in the forward pass and the specified backward precision.}
\end{figure}
\FloatBarrier

\subsection{Training Perplexity for GPT 1.3B on 200 Billion Tokens}

This section contains the full 210B token GPT 1.3B run referenced in Section \ref{sec:experiments} of the main body.
There is an approximately 0.1 validation perplexity gap between MXFP4+RHT only (10.02 ppl) and BF16 (9.92 ppl), whereas MXFP4+RHT+SR matches BF16 (9.90 ppl). 
This suggests that stochastic rounding is important for near-lossless full-scale FP4 training.

\begin{figure}[h!]
\includegraphics[width=\linewidth]{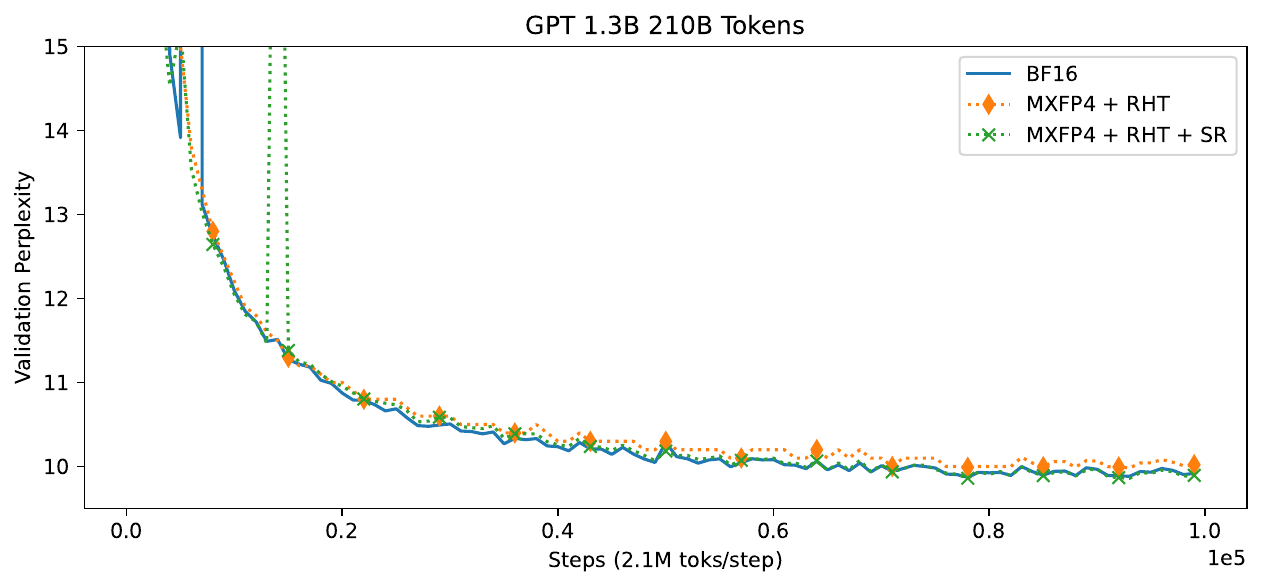}
\caption{Validation perplexity for training GPT 1.3B for 210 billion tokens. All experiments used BF16 in the forward pass and the specified backward precision.}
\end{figure}

\begin{figure}[h!]
\includegraphics[width=\linewidth]{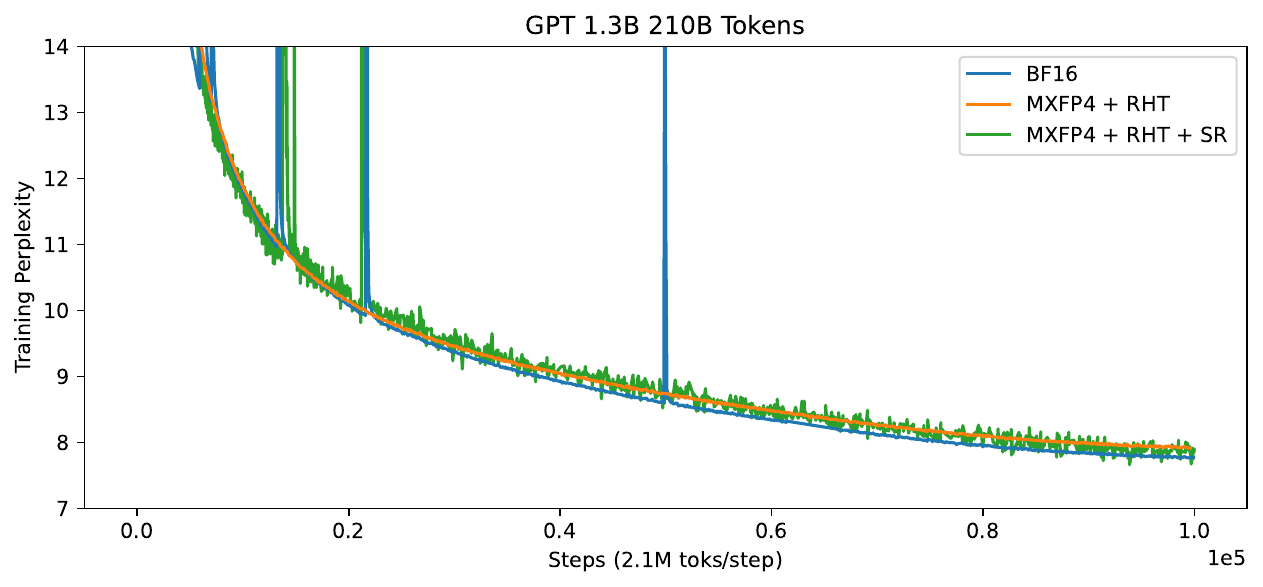}
\caption{Training perplexity for training GPT 1.3B for 210 billion tokens. All experiments used BF16 in the forward pass and the specified backward precision. The BF16 and MXFP4+RHT curves have lower variance due to a configuration difference with the logger.}
\end{figure}
\FloatBarrier

\subsection{Training Curve for GPT 6.7B}

\FloatBarrier
\begin{figure}[h!]
\includegraphics[width=\linewidth]{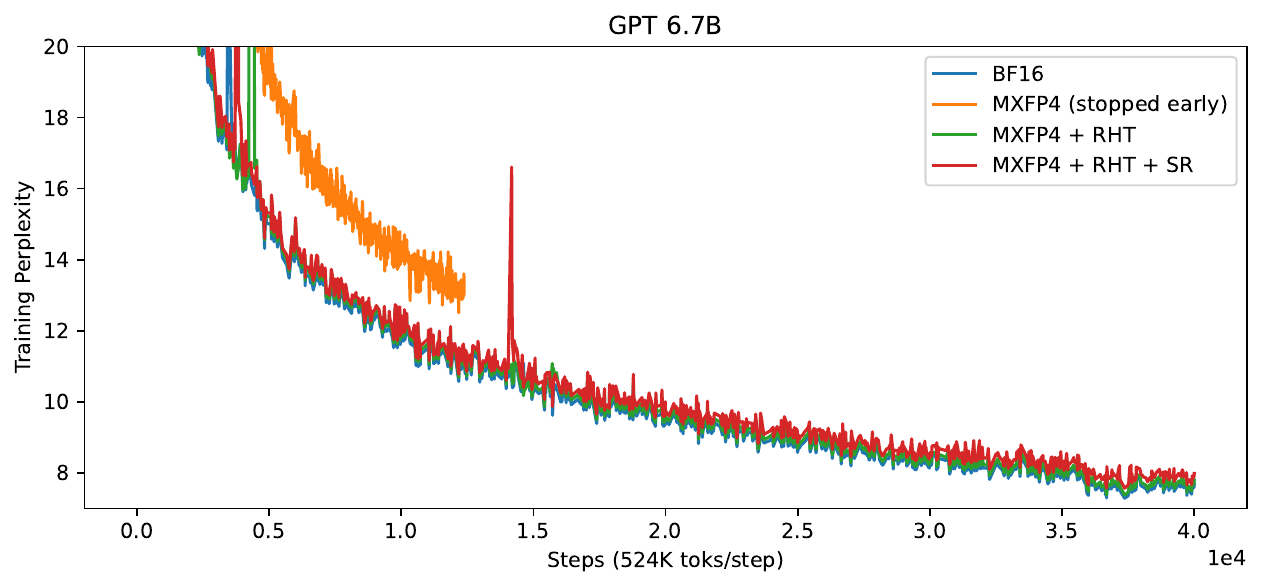}
\caption{Training perplexity for training GPT 6.7B for 20 billion tokens. All experiments used BF16 in the forward pass and the specified backward precision.}
\end{figure}
\FloatBarrier

\section{Experimental Setup Details}

All experiments were run on AWS P4 and G6e EC2 instances.
Our code was based off of the Megatron-LM codebase at Github commit \texttt{a4ad305d4b117217141730b9b18af52dda069450} and the Microsoft microxcaling codebase at Github commit \texttt{7bc41952de394f5cc5e782baf132e7c7542eb4e4}.
We used the NVIDIA Pytorch + Ubuntu 24.04 docker image, which contains a version of Transformer Engine 1.5 for the FP8 experiments.
All models were trained with the AdamW optimizer, FlashAttention \citep{fa1}, and the following hyperparameters:

\begin{table}[h]
\centering
\begin{tabular}{@{}cccc@{}}
\toprule
Hyperparameter             & GPT 345M & GPT 1.3B & GPT 6.7B \\ \midrule
Decoder Layers             & 24       & 24       & 32       \\
Hidden Size                & 1024     & 2048     & 4096     \\
Attention Heads            & 16       & 16       & 32       \\
Context Length             & 1024     & 2048     & 2048     \\
Max. Positional Embeddings & 1024     & 2048     & 2048     \\
Batch Size                 & 64       & 1024     & 256      \\
Learning Rate (LR)         & 0.00015  & 0.0002   & 0.00012  \\
Training Iterations        & 500000   & 20000    & 40000    \\
LR Scheduler               & Cosine   & Cosine   & Cosine   \\
LR Decay Iterations        & 320000   & 20000    & 40000    \\
Minimum LR                 & 1e-5     & 2e-5     & 1.2e-5   \\
Weight Decay               & 1e-2     & 0.1      & 0.1      \\
LR Warmup Fraction         & 0.01     & 0.01     & 0.01     \\
Gradient Clipping          & 1.0      & 1.0      & 1.0      \\ \bottomrule
\end{tabular}
\end{table}

\FloatBarrier
\section{Proof of Lemma 3.1}

\srlemma*

\begin{proof}
First, we show that Algorithm \ref{alg:sr2mx} produces an unbiased MXFP4 estimate of $\frac{3}{4}$ the input vector $v$.
Let $v\in \mathbb{R}^g$, where $g$ is the MX group size.
The input to $\texttt{stochastic\_round\_to\_FP4}$ is given by $w = \frac{3}{4} v/X$, where $X = 2^{\lfloor\log_2(\mbox{argmax}(|v|))\rfloor - 2}$.
Let $m = \mbox{argmax}(|v|)$.
Observe that the largest magnitude element of $w$ is 
$$\frac{3}{4} \frac{m}{2^{\lfloor\log_2(m)\rfloor - 2}} < \frac{3}{4} \frac{m}{2^{\log_2(m) - 3}} = \frac{3}{4}\times 8 = 6$$
By definition, $\texttt{stochastic\_round\_to\_FP4}(x)$ produces an unbiased estimate of $x$ as long as $x$ is ``within range'' -- i.e. it does not overflow outside of the range of representable values in FP4.
Since the maximum normal in FP4 is 6, $\texttt{stochastic\_round\_to\_FP4}(w)$ will give an unbiased FP4 estimate of $\frac{3}{4}v/X$.
Finally, from linearity of expectation, $X * \texttt{stochastic\_round\_to\_FP4}(w)$ gives an unbiased estimate of $\frac{3}{4}v$, as desired.

Now, we show that Algorithm \ref{alg:rhtbw} produces unbiased estimates of $\frac{dL}{dx}$ and $\frac{dL}{dW}$.
Let $C = \texttt{MXFP4\_GEMM}(A, B^T)$, where $A \in \mathbb{R}^{b\times n}$ and $B \in \mathbb{R}^{m \times n}$, and $g | n$.
We have that 
\begin{align}
\mathbb{E}\left[C_{ij}\right] &= \mathbb{E}\left[\sum_{k = 0}^{n/g}\left(X_{A_{i, kg:(k+1)g}}X_{B_{j, kg:(k+1)g}} \sum_{l=0}^g \left(A^{FP4}_{i, kg:(k+1)g}\right)_l \left(B^{FP4}_{j, kg:(k+1)g}\right)_l \right)\right] \\
&= \sum_{k = 0}^{n/g}\left(X_{A_{i, kg:(k+1)g}}X_{B_{j, kg:(k+1)g}} \sum_{l=0}^g \mathbb{E}\left[\left(A^{FP4}_{i, kg:(k+1)g}\right)_l \left(B^{FP4}_{j, kg:(k+1)g}\right)_l \right] \right) 
\end{align}

Where $A_{i, kg:(k+1)g}$ denotes the $k$-th size $g$ vector of the $i$-th row of $A$, $X_{A_{i, kg:(k+1)g}}$ is the scale of applying Algorithm \ref{alg:sr2mx} to $A_{i, kg:(k+1)g}$, and $A^{FP4}_{i, kg:(k+1)g}$ is the FP4 component of applying Algorithm \ref{alg:sr2mx} to $A_{i, kg:(k+1)g}$.
Since stochastic rounding is implemented with independent noise, $A^{FP4}_{i, kg:(k+1)g}$ and $B^{FP4}_{j, kg:(k+1)g}$ are independent random variables.
Thus, 
\begin{align}
\mathbb{E}\left[C_{ij}\right] &= \sum_{k = 0}^{n/g}\left(X_{A_{i, kg:(k+1)g}}X_{B_{j, kg:(k+1)g}} \sum_{l=0}^g \mathbb{E}\left[\left(A^{FP4}_{i, kg:(k+1)g}\right)_l \left(B^{FP4}_{j, kg:(k+1)g}\right)_l \right] \right) \\
&= \sum_{k = 0}^{n/g}\left(X_{A_{i, kg:(k+1)g}}X_{B_{j, kg:(k+1)g}} \sum_{l=0}^g \mathbb{E}\left[\left(A^{FP4}_{i, kg:(k+1)g}\right)_l\right] \mathbb{E}\left[\left(B^{FP4}_{j, kg:(k+1)g}\right)_l \right] \right) \\
&= \sum_{k = 0}^{n/g}\left(X_{A_{i, kg:(k+1)g}}X_{B_{j, kg:(k+1)g}} \sum_{l=0}^g\frac{3}{4}\frac{\left(A_{i, kg:(k+1)g}\right)_l}{X_{A_{i, kg:(k+1)g}}} \frac{3}{4}\frac{\left(B_{j, kg:(k+1)g}\right)_l}{X_{B_{j, kg:(k+1)g}}} \right)  \\
&= \frac{9}{16} \sum_{h=0}^{n} A_{ih}B_{jh} = \frac{9}{16} (AB^T)_{ij}
\end{align}

For $\frac{dL}{dx}$, $A = \frac{dL}{dy} \mbox{diag}(S) H$ and $B = W^T \mbox{diag}(S) H$, where $H$ is the block-diagonal ``small'' Hadamard matrix constructed in Section \ref{sec:smallrht}.% \tao{missing ref here}.
Here, $\mathbb{E}\left[\texttt{MXFP4\_GEMM}(A, B^T)\right] = \frac{9}{16} \frac{dL}{dy} \mbox{diag}(S) H H^T \mbox{diag}(S) W = \frac{9}{16}\frac{dL}{dy} W$.
For $\frac{dL}{dW}$, $A = \frac{dL}{dy}^T \mbox{diag}(S) H$ and $B = x^T \mbox{diag}(S) H$.
Here, $\mathbb{E}\left[\texttt{MXFP4\_GEMM}(A, B^T)\right] = \frac{9}{16} \frac{dL}{dy}^T \mbox{diag}(S) H H^T \mbox{diag}(S) x = \frac{9}{16}\frac{dL}{dy}^T x$.
Finally, scaling both values by 16/9 in lines 10 and 11 gives the desired unbiased gradient estimators.

\end{proof}

\section{Bounding the variance of SR with the RHT}
\rhtvar*

\begin{proof}
Consider two vectors of size $b$: $A, B \in \mathbb{R}^{b}$. 
Then, the output of Algorithm \ref{alg:sr2mx} on $A$ is a scale $X_A$ and vector $Q_A$ such that $\mathbb{E}[{Q_A}_i] = A_i/X_A$ and the expectation is taken over runs of Algorithm \ref{alg:sr2mx}.
Likewise, the output of Algorithm \ref{alg:sr2mx} on $B$ is a scale $X_B$ and vector $Q_B$ s.t. $\mathbb{E}[{Q_B}_i] = B_i/X_B$.
Let $C = X_AX_B\sum_{i = 1}^{b} {Q_A}_i{Q_B}_i$.
Since stochastic rounding is implemented with dithering on $A$ and $B$ with independent random noise,
\begin{align}
\mbox{Var}(C) &= X_AX_B\sum_{i = 1}^{b} \mbox{Var}({Q_A}_i{Q_B}_i) \\
&= X_AX_B\left(\sum_{i=1}^{b} \mbox{Var}({Q_A}_i)\mbox{Var}({Q_B}_i) + \mbox{Var}({Q_A}_i)\mbox{E}({Q_B}_i)^2 + \mbox{Var}({Q_B}_i)\mbox{E}({Q_A}_i)^2\right) \\
&= X_AX_B\left(\sum_{i=1}^{b} \mbox{Var}({Q_A}_i)\mbox{Var}({Q_B}_i) + \mbox{Var}({Q_A}_i)\left(\frac{B_i}{X_B}\right)^2 + \mbox{Var}({Q_B}_i)\left(\frac{A_i}{X_A}\right)^2\right) \\ 
&= \sum_{i=1}^{b} X_AX_B\mbox{Var}({Q_A}_i)\mbox{Var}({Q_B}_i) + \mbox{Var}({Q_A}_i)\left(\frac{X_AB_i^2}{X_B}\right) + \mbox{Var}({Q_B}_i)\left(\frac{X_BA_i^2}{X_A}\right).
\end{align}

Let $\alpha = A_i/X_A$. 
Since ${Q_A}_i$ is the output of stochastic rounding to FP4, ${Q_A}_i$ takes on values $f(\alpha)$ with probabilty $\frac{c(\alpha) - \alpha}{c(\alpha) - f(\alpha)}$ and $c(\alpha)$ with probability $\frac{\alpha - f(\alpha)}{c(\alpha) - f(\alpha)}$, where $f(\alpha)$ denotes the largest representable FP4 value $\le \alpha$ and $c(\alpha)$ denotes the smallest representable FP4 value $\ge \alpha$.
Observe that $f(\alpha)$ and $c(\alpha)$ are both guaranteed to exist due to line 4 in Algorithm \ref{alg:sr2mx}. 
Then, 

\begin{align}
\mbox{Var}({Q_A}_i) &= \frac{f(\alpha)^2(c(\alpha) - \alpha) + c(\alpha)^2(\alpha - f(\alpha))}{c(\alpha) - f(\alpha)} - \alpha^2 \\
&= \frac{(c(\alpha)^2 - f(\alpha)^2)\alpha + (f(\alpha) - c(\alpha))f(\alpha)c(\alpha)}{c(\alpha) - f(\alpha)} - \alpha^2 \\
&= (c(\alpha) + f(\alpha))\alpha - f(\alpha)c(\alpha) - \alpha^2.
\end{align}

Let $\delta^+ = c(\alpha) - \alpha$ and $\delta^- = f(\alpha) - \alpha$. 
Then, %\tao{all these variances at the left side of equation 16-22 should be $\mbox{Var}({Q_A}_i)$ not $\mbox{Var}(\alpha)$, right?} \albert{yes}

\begin{align}
\mbox{Var}({Q_A}_i) &= (c(\alpha) + f(\alpha))\alpha - f(\alpha)c(\alpha) - \alpha^2 \\
&= (2\alpha + \delta^- + \delta^+)\alpha - (\alpha + \delta^-)(\alpha + \delta^+) - \alpha^2 \\
&= -\delta^-\delta^+ = (c(\alpha) - \alpha)(\alpha - f(\alpha)) \\
&= \mathcal{O}((c(\alpha) - f(\alpha))^2).
\end{align}

Since $c(\alpha) - f(\alpha)$ is $O(\Delta)$, 
\begin{align}
\mbox{Var}(C) &= \mathcal{O}\left(b\Delta^4X_AX_B + \Delta^2\frac{X_A}{X_B} \sum_{i=1}^{b} B_i^2 + \Delta^2\frac{X_B}{X_A} \sum_{i=1}^{b} A_i^2\right) \\
&= \mathcal{O}\left(b\Delta^4X_AX_B + \Delta^2\frac{X_A}{X_B} \|B\|^2 + \Delta^2\frac{X_B}{X_A} \|A\|^2 \right).
\end{align}

Since $X_A = \Theta(\|A\|_\infty)$ and likewise for $B$, this reduces to 
\begin{align}
\mbox{Var}(C) &= \mathcal{O}\left(b\Delta^4\|A\|_\infty \|B\|_\infty + 2b\Delta^2\|A\|_\infty \|B\|_\infty \right)\\
&= \mathcal{O}(b\Delta^4\|A\|_\infty \|B\|_\infty)
\end{align}
% \tao{why? $X_A$, the $shared_exp$ has a $-emax$ term, with the float to value representation of it, this more like to be $\Theta(1)$, in any case, it can't be $\mathcal{O}(\|A\|_\infty)$, we can discuss $\Theta(1)$ or $\Theta(\|A\|_\infty)$}\tao{problematic, connected to last comment}.
% \albert{$X_A$ is the scale from mx quantization. It is the largest magnitude entry in the vector scaled s.t. the largest element after dividng by $X_A$ is the largest normal (roughly) in FP4. I guess it should be $\Theta(||A||_\infty)$ then.}
If $A$ and $B$ are transformed by the RHT in the way Algorithm X does (i.e. $\tilde A \gets ASH^T$ and $\tilde B \gets HSB$), then we can bound $\|\tilde A\|_\infty$ and $\|\tilde B\|_\infty$. %\tao{this should be bound on $\|ASH^T\|_\infty$ and $\|HSB\|_\infty$}. 
From \cite{qs}, $\forall i, 1 \le i \le b$, 

\begin{equation}
\mathbb{P}(|e_i\tilde A| \ge \epsilon) = \mathbb{P}(|e_iASH^T| \ge \epsilon) \le 2 \exp \left(\frac{-\epsilon^2 b}{2\|A\|^2}\right).
\end{equation}

From the union bound, 

\begin{align}
\mathbb{P}\left(\max_i |e_iASH^T| \ge \epsilon\right) &\le 2b \exp \left(\frac{-\epsilon^2 b}{2\|A\|^2}\right) \\
\mathbb{P}\left(\max_i |e_iASH^T| \ge \sqrt{\frac{2\|A\|^2}{b} \log \left(\frac{2b}{\epsilon}\right)} \right) &\le \epsilon.
\end{align}

so with probability $\ge 1-\epsilon$, %\tao{the bound holds for $\|ASH^T\|_\infty$ first, then connects to $X_A$, the rest follows if it's $\Theta(\|ASH^T\|_\infty)$, not big O}

\begin{equation}
\|A\|_\infty = \mathcal{O}\left(\sqrt{\frac{2\|A\|^2}{b} \log \left(\frac{2b}{\epsilon}\right)}\right)
\end{equation}

and with probability $\ge (1-\epsilon)^2$,

\begin{equation}
\mbox{Var}(C) = \mathcal{O}\left(\Delta^4 \|A\|\|B\| \log\left(\frac{2b}{\epsilon}\right)\right).
\end{equation}
\end{proof}

\end{document}